\documentclass{article}

\usepackage[preprint]{neurips_2024}

\usepackage[utf8]{inputenc} 
\usepackage[T1]{fontenc}    
\usepackage{hyperref}       
\usepackage{url}            
\usepackage{booktabs}       
\usepackage{amsfonts}       
\usepackage{nicefrac}       
\usepackage{microtype}      
\usepackage{xcolor}         

\usepackage{lineno}
\usepackage{setspace}
\usepackage{capt-of}
\usepackage{mathtools}
\usepackage{mathrsfs}
\usepackage[normalem]{ulem}
\usepackage{enumitem}
\usepackage{amsmath}
\usepackage{amsthm}
\usepackage{amssymb}
\usepackage{ulem}
\usepackage{bbm}
\usepackage{xcolor}
\usepackage[ruled]{algorithm2e}

\SetCommentSty{mycommfont}
\usepackage{subcaption}
\usepackage{dirtytalk}
\usepackage{tikz}
\usetikzlibrary{shapes.geometric, arrows.meta, shadows, positioning}
\usepackage{tikz-cd}
\usepackage{nicematrix}
\usepackage{bm}
\usepackage{multirow}
\usepackage{longtable}
\usepackage{float}
\usepackage{comment}
\usepackage{diagbox}
\usepackage{hyperref}
\newtheorem{theorem}{Theorem}[section]

\newtheorem{lemma}[theorem]{Lemma}
\newtheorem{proposition}[theorem]{Proposition}
\usepackage{array}
\usepackage{booktabs}       
\usepackage{multirow}     
\usepackage{dsfont} 
\usepackage{caption}
\usepackage{wrapfig}
\title{Risk-sensitive Actor-Critic with Static Spectral Risk Measures for Online and Offline Reinforcement Learning} 

\author{%
  Mehrdad Moghimi, Hyejin Ku \\
  Department of Mathematics and Statistics\\
  York University\\
  Toronto, Ontario, Canada \\
  \texttt{\{moghimi,hku\}@yorku.ca} \\
}

\begin{document}

\maketitle

\begin{abstract}
The development of Distributional Reinforcement Learning (DRL) has introduced a natural way to incorporate risk sensitivity into value-based and actor-critic methods by employing risk measures other than expectation in the value function. While this approach is widely adopted in many online and offline RL algorithms due to its simplicity, the naive integration of risk measures often results in suboptimal policies. This limitation can be particularly harmful in scenarios where the need for effective risk-sensitive policies is critical and worst-case outcomes carry severe consequences. To address this challenge, we propose a novel framework for optimizing static Spectral Risk Measures (SRM), a flexible family of risk measures that generalizes objectives such as CVaR and Mean-CVaR, and enables the tailoring of risk preferences. Our method is applicable to both online and offline RL algorithms. We establish theoretical guarantees by proving convergence in the finite state-action setting. Moreover, through extensive empirical evaluations, we demonstrate that our algorithms consistently outperform existing risk-sensitive methods in both online and offline environments across diverse domains. 
\end{abstract}

\begingroup
\let\clearpage\relax

\section{Introduction}
Risk management plays a crucial role in sequential decision-making tasks across various domains, including finance, healthcare, and robotics. To address these risks, many approaches change the standard objective of maximizing expected returns to alternative risk measures, such as the Conditional Value-at-Risk (CVaR) \citep{Bauerle.Ott2011} or coherent risk measures \citep{Tamar.etal2017}. Another strategy is to impose constraints, such as variance-based limits \citep{Tamar.etal2012a} or dynamic risk measures \citep{Chow.Pavone2013}, to control exposure to worst-case scenarios.

Recently, Distributional Reinforcement Learning (DRL) has gained significant attention as a framework for risk-sensitive RL (RSRL) by estimating the entire distribution of returns rather than focusing solely on the expected value \citep{Morimura.etal2010a,Bellemare.etal2017a}. This approach offers a straightforward method for risk mitigation by employing a risk measure, such as a Mean-Variance or Distortion risk measure, instead of the expected value \citep{Dabney.etal2018b, Ma.etal2020}. However, within the DRL framework, applying a fixed risk measure at each time step does not necessarily result in policies that optimize static or dynamic risk measures over the entire trajectory \citep{Lim.Malik2022}. With such iterative risk measures, action selection at different states may deviate from the agent’s overall risk preferences, potentially leading to suboptimal policies. For example, optimizing $\operatorname{CVaR}_{0.1}$ at a future state may not align with optimizing the worst 10\% of returns from the initial state. This phenomenon, referred to as time inconsistency, is a well-documented challenge in risk-sensitive decision-making \citep{Shapiro.etal2014}.

Due to the simplicity of incorporating iterative risk measures with DRL, this approach has also been adopted in offline RL \citep{Urpi.etal2021a, Ma.etal2021}. In offline RL, policies are learned from previously collected data \citep{Levine.etal2020} and risk-sensitive offline RL with iterative risk measures not only face the same issues observed in the online setting but also encounter additional challenges unique to offline RL, such as the inability to explore new high-reward states outside the dataset and the distributional shift problem, where the policy is trained on one state-action distribution but evaluated on another. To address these issues, existing offline RL methods often employ techniques like policy constraints \citep{Fujimoto.etal2019a, Peng.etal2019, Nair.etal2021} and value function regularization \citep{Kumar.etal2020a}.  

\begin{wrapfigure}{r}{0.4\textwidth}
  \begin{center}
  \vspace*{-15pt}
    \includegraphics[width=1.0\linewidth]{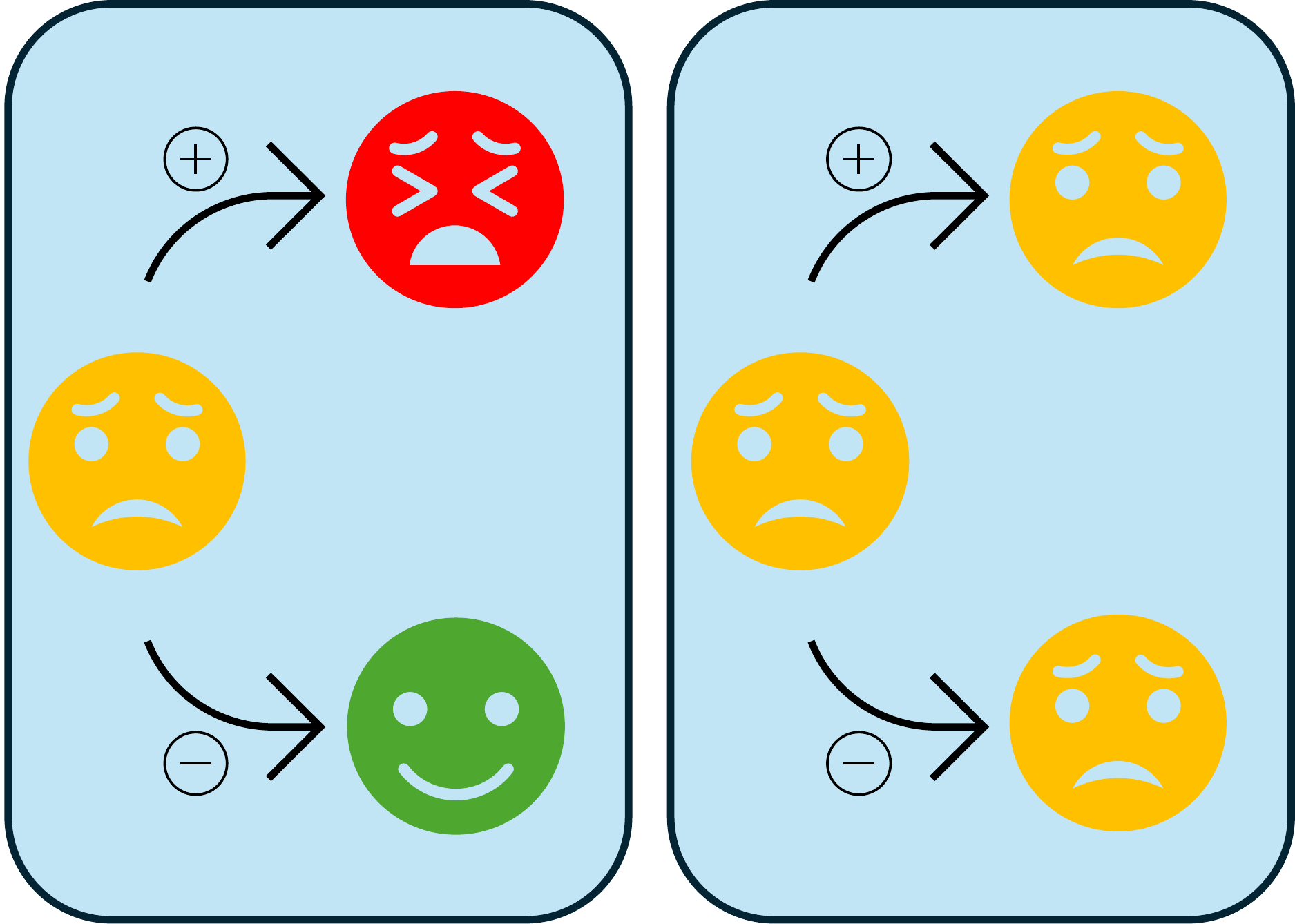}
  \end{center}
  \caption{The risk-sensitivity changes implicitly with Static risk measures \textbf{(Left)}. However, Iterative risk measures have the same risk-sensitivity regardless of the observed path \textbf{(Right)}.}
\end{wrapfigure}

Optimizing static risk measures offers a more interpretable solution by focusing on policies that maximize worst-case returns. Unlike dynamic or iterative risk measures, which assume that worst-case scenarios can occur at each step, static risk measures impose a limit on the worst-case scenarios that can occur over an entire episode \citep{Chow.etal2015}. Consequently, with static risk measures, the risk level of the agent changes implicitly based on the observed path \citep{Pflug.Pichler2016, Moghimi.Ku2025a}. While this behavior may seem unintuitive, it reflects a key feature of static measures: after observing a favorable sequence of events, the agent becomes more cautious in subsequent steps to respect the overall bound on worst-case scenarios.

Although optimizing static CVaR has been the focus of many studies due to its simple formulation, one of its major drawbacks is that it exclusively focuses on the left tail of the return distribution, ignoring potential upsides. This limitation, often referred to as “Blindness to Success” \citep{Greenberg.etal2022a}, can result in suboptimal policies. To address this issue, as well as the previously mentioned challenges, we introduce a simple yet powerful actor-critic framework for optimizing static Spectral Risk Measures (SRM) that can be applied in both offline and online settings. By breaking down the steps required to optimize an SRM, we demonstrate how such a framework can be derived. Moreover, since SRM is a flexible risk measure that encompasses various measures, such as Mean-CVaR and the Exponential risk measure, this framework fully leverages the potential of DRL.  

In scenarios where avoiding worst-case outcomes is critical and collecting new data is costly, learning effective risk-sensitive policies from offline datasets is especially valuable, as it eliminates the need for environment interaction. Despite this importance, existing work on risk-sensitive offline reinforcement learning has largely focused on iterative risk measures. To the best of our knowledge, the optimization of static risk measures remains unexplored in the offline setting. A further advantage of learning risk-sensitive policies offline is the ability to tailor policies to different levels of risk aversion. In this regard, the flexibility of static spectral risk measures enables the learning of a diverse set of policies adapted to varying risk preferences. As we demonstrate in our experiments, our algorithms not only achieve strong performance under risk-sensitive criteria but also allow for the customization of policies to suit specific risk profiles.

Our framework can be utilized with both stochastic and deterministic policies. Stochastic policies are well-suited for environments with uncertainty and imperfect information due to their inherent randomness \citep{Sutton.Barto2018a}. However, this randomness can degrade the performance of risk-averse policies, and the exploration advantages of stochastic policies are lost in the offline setting \citep{Urpi.etal2021a}. This motivates the use of deterministic policies in our framework. For this purpose, we utilize TD3 \citep{Fujimoto.etal2018} and TD3BC \citep{Fujimoto.Gu2021a} as high-performing online and offline RL algorithms with deterministic policies. Accordingly, in addition to our stochastic policy-based methods, AC-SRM for online learning and OAC-SRM for offline learning, we introduce TD3-SRM and TD3BC-SRM as their deterministic counterparts for online and offline settings, respectively. Our empirical results demonstrate that optimizing static SRM not only addresses theoretical concerns but also produces policies that outperform other risk-sensitive offline algorithms across various environments.  

In summary, our key contributions are as follows:
\begin{itemize}
    \item We propose a risk-sensitive actor-critic framework for optimizing the static Spectral Risk Measures, which is applicable to both online and offline learning.
    \item We prove the convergence of our algorithm in the online MDP setting with finite state and action spaces. 
    \item We perform a comprehensive empirical analysis of our framework with both stochastic and deterministic policies in online and offline settings, demonstrating its overall effectiveness and superiority compared to existing methods.
\end{itemize}

\section{Related Work}

\subsection{Risk-sensitive Online RL}
Risk-sensitive reinforcement learning (RL) focuses on identifying policies that maximize risk-adjusted returns, formulated as $\max_{\pi \in \boldsymbol{\pi}} \rho(Z^\pi)$. In value-based methods, several approaches have been proposed: \citet{Bauerle.Ott2011} and \citet{Chow.etal2015} study the Conditional Value at Risk (CVaR), while \citet{Bauerle.Rieder2014} focus on utility functions. In the distributional RL framework, \citet{Dabney.etal2018b} explore the use of risk measures beyond expected value, such as distortion risk measures, for action selection. Further developments include optimizing static CVaR, as shown in \citet{Bellemare.etal2023} and \citet{Lim.Malik2022}, and general statistical functionals of the return distribution, as discussed by \citet{Pires.etal2025}. In policy gradient methods, \citet{Chow.etal2018a} address CVaR-constrained problems, \citet{Tamar.etal2012a} study variance-based objectives, and \citet{Tamar.etal2015a} focus on optimizing CVaR. Dynamic coherent and convex risk measures have also been considered by \citet{Tamar.etal2017} and \citet{Coache.Jaimungal2023}, respectively. Additionally, \citet{Bisi.etal2022} propose a state-augmentation approach for optimizing risk measures such as mean-variance and utility functions, and \citet{Ma.etal2020} incorporate risk sensitivity into SAC using a distributional critic.

There have been a few works that address the optimization of static SRM. \citet{Bauerle.Glauner2021a} decompose the problem into two parts: identifying a risk function corresponding to the SRM, and then optimizing a policy under that risk function. They use a value-based method for policy optimization and rely on global optimization to learn the risk function. \citet{Kim.etal2024a} use SRM as the constraint of the optimization problem and employ a gradient-based method to learn the parameters of the risk function, which can be computationally expensive. More recently, \citet{Moghimi.Ku2025a} propose a value-based approach for optimizing SRM using distributional RL, and use a closed-form solution to compute the risk function. In this paper, we follow this approach due to its simplicity and strong performance and extend it by developing an actor-critic algorithm for both online and offline settings.

\subsection{Offline RL} 
Offline RL, or batch RL, involves learning policies from static datasets without further interaction with the environment. Addressing the key challenge of \textit{distributional shift}, where the state-action distribution of the dataset differs from that of the learned policy, has led to a variety of approaches. These include policy constraint methods, such as Batch-Constrained Q-learning (BCQ) \citep{Fujimoto.etal2019a} and Advantage-Weighted Actor-Critic (AWAC) \citep{Nair.etal2021}, which constrain the learned policy to stay close to the behavior policy and reduce the risk of poor decisions in underexplored areas. Regularization techniques, like Conservative Q-Learning (CQL) \citep{Kumar.etal2020a}, address overestimation of actions not well supported by the dataset. Additionally, ensemble methods \citep{Agarwal.etal2020} improve robustness by aggregating multiple value estimates. These methods represent a few of the many strategies proposed to tackle offline RL challenges. For a comprehensive review, we refer readers to the surveys by \citet{Levine.etal2020} and \citet{Prudencio.etal2023}.

\subsection{Risk-sensitive Offline RL}
Despite the importance of learning risk-averse policies from datasets, only a few works have addressed this problem. The first work in this area is the Offline Risk-Averse Actor-Critic (ORAAC) algorithm \citep{Urpi.etal2021a} which extends the BCQ algorithm \citep{Fujimoto.etal2019a} with a distributional critic and uses risk-sensitive action-selection. Similarly, the Conservative Offline Distributional Actor-Critic (CODAC) algorithm \citep{Ma.etal2021} extends the CQL algorithm \citep{Kumar.etal2020a} with a distributional critic and learns risk-averse policies by utilizing the return-distribution of state-action pairs. \citet{Bai.etal2022} propose a new architecture called Monotonic Quantile Network (MQN) to improve the learning of the distributional value function in the offline setting and uses the CVaR of this value function to learn risk-averse policies. Finally, \citet{Rigter.etal2023} proposes an actor-critic algorithm that optimizes a dynamic risk measure, which requires learning the MDP model from the dataset. While model-based algorithms offer improved sample efficiency in the online setting by leveraging learned environment dynamics, they are computationally expensive, and the model bias can lead to compounding errors over multiple time-steps, which can degrade policy performance.

\section{Preliminary Studies}

In this work, we consider a Markov Decision Process (MDP) defined by the tuple \((\mathcal{X}, \mathcal{A}, \mathcal{R}, \mathcal{P}, \gamma, \xi_0)\), where \(\mathcal{X}\) and \(\mathcal{A}\) represent the state and action spaces. The reward function \(\mathcal{R}: \mathcal{X} \times \mathcal{A} \rightarrow \mathscr{P}(\mathbb{R})\) specifies a distribution over rewards, while the transition function \(\mathcal{P}: \mathcal{X} \times \mathcal{A} \rightarrow \mathscr{P}(\mathcal{X})\) defines the probability of moving to the next state. The initial state follows the distribution \(\xi_0\), and \(\gamma \in [0,1)\) is the discount factor. We assume that rewards are bounded within \([R_{\min}, R_{\max}]\), where \(R_{\min} \geq 0\).  

A Markov stochastic policy is a mapping \(\pi: \mathcal{X} \rightarrow \mathscr{P}(\mathcal{A})\) that assigns a probability distribution over actions given a state. We denote the space of Markov policies as $\Pi$ and state occupancy measure under policy \(\pi\) as \(d_{\xi_0}^{\pi}\). The discounted return when starting from state \(x\), taking action \(a\), and following policy \(\pi\) is defined as  
\[
G^\pi(x, a) := \sum_{t=0}^{\infty} \gamma^t R_t,
\]  
where rewards \(R_t\) are drawn from \(\mathcal{R}(X_t, A_t)\), actions \(A_t\) are sampled from \(\pi(\cdot | X_t)\), and next states \(X_{t+1}\) follow the transition dynamics \(\mathcal{P}(X_t, A_t)\). The return from state \(x\) under policy \(\pi\) is denoted as \(G^\pi(x)\), while \(G^\pi=\mathbb{E}_{x_0\sim\xi_0}\left[G^\pi(x_0)\right] \) represents the initial state's return distribution.  

For a parameterized policy \(\pi_\theta\), the value of the policy is defined as  
\[
J(\pi_\theta) = \mathbb{E}\left[G^{\pi_\theta}\right],
\]  
and our goal is to find the optimal parameters \(\theta^*\) that maximize this value:  
\[
\theta^* = \arg\max_\theta J(\pi_\theta).
\]  
The Policy Gradient Theorem \citep{Sutton.etal1999} provides a way for updating the policy \(\pi_\theta\) by computing the gradient of the performance objective \(J(\pi_\theta)\), given by:
\begin{equation}
\label{eq:gradient}
    \nabla_\theta J(\pi_\theta) = \mathbb{E}_{x\sim d_{\xi_0}^{\pi_\theta},\; a\sim \pi_\theta}\left[\nabla_\theta\log\pi_\theta(a\mid x) Q^{\pi_\theta}(x, a)\right],
\end{equation}
where $d_{\xi_0}^{\pi}(x) = \mathbb{E}_{x_0\sim\xi_0}\left[d_{x_0}^{\pi}(x)\right]$ is the discounted state visitation distribution under initial distribution $\xi_0$ and $d_{x_0}^{\pi}(x)$ is defined as $d_{x_0}^{\pi}(x):= (1-\gamma)\sum_{t=0}^{\infty}\gamma^t\operatorname{Pr}^{\pi}(x_t=x\mid x_0)$ where $\operatorname{Pr}^{\pi}(x_t=x\mid x_0)$
is the probability of visiting state $x$ after starting at $x_0$ and executing $\pi$. The Q-value function \(Q^{\pi}(x, a)\) represents the expected return after taking action \(a\) in state \(x\):  
\[
Q^{\pi}(x,a) = \mathbb{E}\left[G^\pi(x, a)\right].
\]
On-policy learning methods update the policy using data collected by executing the current policy in the environment. In contrast, off-policy learning allows learning about a target policy \(\pi_\theta\) using data generated by a different behavior policy. 
This enables more sample-efficient learning by reusing past experiences. 

In online off-policy learning, experiences collected through ongoing interaction are stored in a replay buffer, from which the agent samples to update its policy. In contrast, offline RL (also known as batch RL) involves learning solely from a fixed dataset of transitions generated by an unknown behavior policy. In this setting, the agent has no further interaction with the environment, and the key challenge is to learn a high-performing policy purely from static data.

\subsection{Distributional RL}
\label{sec:distributional}
Distributional reinforcement learning \citep{Morimura.etal2010a, Bellemare.etal2017a} extends standard RL by modeling the full distribution of returns \( G^\pi(x, a) \), denoted as \( \eta^\pi(x, a) \). This return distribution is the unique fixed point of the distributional Bellman operator 
\(
\mathcal{T}^{\pi}: \mathscr{P}(\mathbb{R})^{\mathcal{X} \times \mathcal{A}}\rightarrow \mathscr{P}(\mathbb{R})^{\mathcal{X} \times \mathcal{A}},
\)
defined as  
\[
\left(\mathcal{T}^{\pi}\eta\right)(x, a) = \mathbb{E}_\pi\left[(b_{R,\gamma})_{\#} \eta(X^\prime, A^\prime) \mid X=x, A=a\right],
\]
where \( A^\prime \sim \pi(\cdot) \) and \( b_{r,\gamma}(z) = r + \gamma z \). The push-forward distribution \( (b_{R,\gamma})_{\#} \eta(X^\prime, A^\prime) \) represents the return distribution of  \( (b_{R,\gamma}) G(X^\prime, A^\prime) \). The operator can also be written in terms of return variables:  
\[
G^{\pi}(x, a) = R(x, a) + \gamma \mathbb{E}_{x^{\prime} \sim \mathcal{P}(x^{\prime} \mid x, a), a^{\prime} \sim \pi(x^{\prime})} \left[G^{\pi}(x^{\prime}, a^{\prime})\right].
\]  

In this work, we model the return distribution using quantiles and update our estimates via quantile regression, following the QR-DQN approach \citep{Dabney.etal2018a}. Given cumulative probabilities \( \tau_i = i/N \) for \( i = 0, \dots, N \), the return distribution is parameterized as  
\[
\eta_q(x, a) = \frac{1}{N} \sum_{i=1}^{N} \delta_{q_i(x, a)},
\]  
where the support points are defined by  
\[
q_i(x, a) = F_{G(x, a)}^{-1}(\hat{\tau}_i), \quad \hat{\tau}_i = (\tau_{i-1} + \tau_i)/2, \quad 1 \leq i \leq N.
\]  

To estimate the quantiles of a distribution \(\nu\), we minimize the quantile regression Huber loss, defined as: 
\begin{equation}
\label{huberloss}
    \mathcal{L}(G, Z) = \sum_{i=1}^{N} \mathbb{E}_{Z \sim \nu} \left[l_{\hat{\tau}_i}^{\kappa}(Z - q_i)\right],
\end{equation}
where  \(l_\tau^\kappa(u) = \left|\tau - \mathbb{I}\{u < 0\}\right| \cdot \mathcal{L}_\kappa^{huber}(u)\) and \(\mathcal{L}_\kappa^{huber}(u)\) denotes the Huber loss \citep{Huber1992}, which combines squared and absolute error to avoid constant gradients near zero:
\[
\mathcal{L}_\kappa^{huber}(u) =
\begin{cases}
\frac{1}{2} u^2, & \text{if } |u| \leq \kappa, \\
\kappa \left(|u| - \frac{1}{2} \kappa\right), & \text{otherwise}.
\end{cases}
\]

\subsection{Risk-sensitive Policies}

The availability of return distributions in actor-critic methods with a distributional critic enables the use of risk measures other than the expectation when estimating state-action values. A simple approach is to apply a risk measure $\rho$ at each step by replacing the Q-values in Equation \ref{eq:gradient} with:  
\[
Q(x,a) =\rho\left(G^\pi(x, a)\right).
\]
This method has been widely used in both online \citep{Dabney.etal2018b, Ma.etal2020} and offline \citep{Urpi.etal2021a, Ma.etal2021} RL for risk-sensitive policy optimization.  

However, as discussed in \citet{Lim.Malik2022}, this iterative risk application can be problematic. It can lead to overly optimistic or overly conservative estimates, and optimizing a policy based on these Q-values does not necessarily align with maximizing either the static risk-adjusted value:  
\[
J(\pi) = \rho\left(G^{\pi}\right),
\]
or the dynamic risk-adjusted value:  
\[
J_d(\pi) = \rho\left(R_0 + \gamma\rho\left(R_1 + \gamma\rho\left(R_2 + \gamma(\cdots)\right)\right)\right).
\]  

Dynamic risk measures, which repeatedly apply a risk function at each step, introduce further challenges. Computing policy gradients in this setting require knowledge of the MDP model to adjust transition probabilities, making both optimization and interpretation more difficult. Additionally, selecting an appropriate risk parameter is nontrivial, as the fixed risk level can lead to overly conservative policies. Given these issues, we instead focus on directly optimizing the static risk-adjusted value.

\subsection{Spectral Risk Measure}
Let \( Z \in \mathcal{Z} \) denote the random variable representing the return, and let \( \rho: \mathcal{Z} \to \mathbb{R} \) be a risk measure, where \( \rho(Z) \) is the risk-adjusted value of \( Z \). The quantile function of \( Z \) is defined as  
\[
F_Z^{-1}(u) = \inf \{z \in \mathbb{R} \mid F_Z(z) \geq u\}, \quad u \in [0, 1],
\]
where \( F_Z(z) = \mathbb{P}(Z \leq z) \) is the cumulative distribution function (CDF). While expectation assigns equal weights to all quantiles, a natural extension is to assign different weights based on risk preferences.  Spectral Risk Measure (SRM), introduced by \citet{Acerbi2002}, formalizes this idea by defining  
\[
\operatorname{SRM}_\phi(Z) = \int_0^1 F_Z^{-1}(u) \phi(u) \, \mathrm{d}u,
\]
where \( \phi: [0, 1] \to \mathbb{R}_{+} \) is a left-continuous, non-increasing risk spectrum function satisfying \( \int_0^1 \phi(u) \, du = 1 \). The function \( \phi(u) \) expresses the risk preference across different quantiles of the return distribution. \footnote{
\citet{Dabney.etal2018b} employ the Distortion Risk Measure (DRM). A coherent DRM, characterized by a concave distortion function $g$, is equivalent to the Spectral Risk Measure where $g^\prime(u)=\phi(u)$ \citep{Henryk.Silvia2006}.} Various risk spectrums have been explored in the literature:

\begin{wrapfigure}{r}{0.3\textwidth}
  \begin{center}
  \vspace*{-130pt}
    \includegraphics[width=1.0\linewidth]{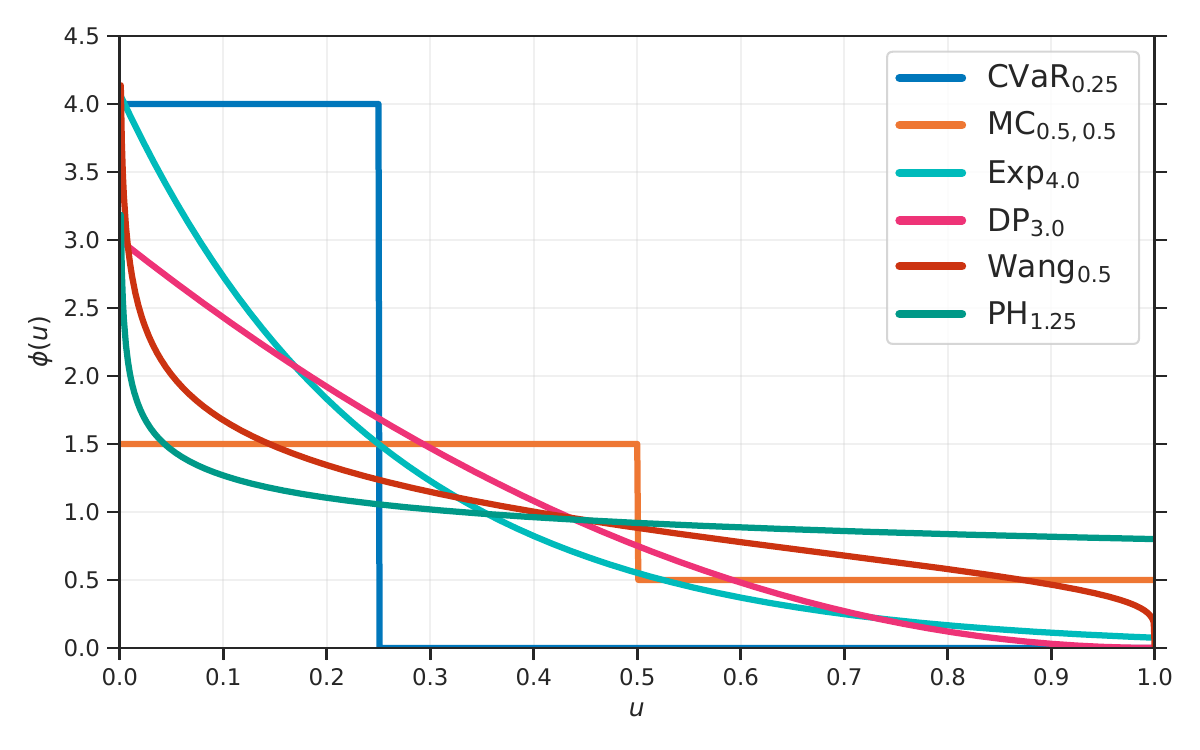}
  \end{center}
  \vspace*{-15pt}
  \caption{Illustration of various risk spectrums.}
\end{wrapfigure}

\begin{itemize}[leftmargin=15pt]
    \item CVaR: \( \phi_{\alpha}(u) = \frac{1}{\alpha} \mathds{1}_{[0, \alpha]}(u) \),
    \item Mean-CVaR (MC): \( \phi_{\alpha, \omega}(u) = \omega \mathds{1}_{[0, 1]}(u) + (1-\omega)\frac{1}{\alpha} \mathds{1}_{[0, \alpha]}(u) \),
    \item Exponential (Exp): \( \phi_\alpha(u) = \frac{\alpha e^{-\alpha u}}{1 - e^{-\alpha}} \),
    \item Dual Power (DP): \( \phi_{\alpha}(u) = \alpha (1 - u)^{\alpha - 1} \),
    \item Wang: \( \phi_{\alpha}(u) = e^{(\alpha \Phi^{-1}(u) - \alpha^2/2)} \),
    \item Proportional Hazard (PH): \( \phi_{\alpha}(u) = \frac{1}{\alpha}u^{(1/\alpha - 1)}  \).
\end{itemize}

where $\Phi^{-1}(u)$ denotes the inverse CDF of a standard Normal distribution \citep{Wang1995, Wang2000}.

Our goal in this work is to find a policy that maximizes its static risk-adjusted value. In the next section, we show that achieving this requires using the supremum representation of SRM, which applies to any SRM with a bounded spectrum:
\begin{equation}
\label{eq:srmg}
\operatorname{SRM}_\phi(Z) = \sup _{h \in \mathcal{H}}\left\{\mathbb{E}\left[h(Z)\right]+\int_0^1 \hat{h}(\phi(u)) \mathrm{d} u\right\}.
\end{equation}
Here, \( \mathcal{H} \) is the set of concave functions, and \( \hat{h} \) is the concave conjugate of \( h \) \citep{Pichler2015}. The supremum is attained by the function \( h_{\phi, Z}: \mathbb{R} \to \mathbb{R} \), given by  
\begin{equation}
\label{functionh}
h_{\phi, Z}(z) = \int_0^1 \left[ F_Z^{-1}(\alpha) + \frac{1}{\alpha} (z - F_Z^{-1}(\alpha))^{-} \right] \mu(\mathrm{d}\alpha),
\end{equation}
where the probability measure \( \mu:[0,1]\to[0,1] \) satisfies  
\begin{equation}
    \label{eq:phimu}
    \phi(u) = \int_u^1 \frac{1}{\alpha} \mu(\mathrm{d}\alpha),
\end{equation}
and \( h_{\phi, Z}(z) \) satisfies  
\[
\int_0^1 \hat{h}_{\phi, Z}(\phi(u)) \, \mathrm{d} u = 0.
\]

\section{Method}
\label{method}

In this section, we present a unified actor-critic framework for optimizing static Spectral Risk Measures (SRM) through a bi-level optimization approach. By leveraging the supremum representation of SRMs, we decompose the objective into an inner loop that optimizes the policy for a fixed risk function and an outer loop that updates the risk function based on the return distribution. This framework supports both online and offline settings, as well as stochastic and deterministic policies. We first introduce AC-SRM for online learning, followed by OAC-SRM for offline learning with policy constraints, and finally extend our method to deterministic policy gradients via TD3-SRM and TD3BC-SRM.

To be more precise, we aim to optimize the static risk-adjusted return of a policy under a Spectral Risk Measure (SRM), given by:  
\[
J(\pi) = \operatorname{SRM}_{\phi} (G^\pi).
\]
We leverage the supremum representation of the SRM (Equation \eqref{eq:srmg}) to reformulate this objective as a bi-level optimization problem. To that end, we define an auxiliary value function for a fixed risk function \( h \in \mathcal{H} \):  
\[
J(\pi, h) := \mathbb{E}\left[h\left(G^\pi\right)\right] + \int_0^1 \hat{h}(\phi(u)) \, \mathrm{d}u.
\]
Using this formulation, the overall optimization problem becomes:  
\begin{equation}
\label{eq:obj}
\max_{\pi \in \Pi} J(\pi) = \max_{\pi \in \Pi} \max_{h \in \mathcal{H}} J(\pi, h) = \max_{h \in \mathcal{H}} \left( \max_{\pi \in \Pi} J(\pi, h) \right).
\end{equation}

For clarity, we omit the subscript \( h \) in \( \pi_h \) throughout the remainder of the paper, and simply write \( \pi \) when the dependence on \( h \) is clear from context.

For the inner optimization, we adopt an actor-critic method with a distributional critic, enabling simultaneous learning of both the optimal policy and its associated return distribution. In this context, the optimal policy is Markovian in an extended state space \(\mathcal{\bar{X}} := \mathcal{X} \times \mathcal{S} \times \mathcal{C}\) \citep{Bauerle.Glauner2021a}, where \(\mathcal{S}\) represents the space of accumulated discounted rewards, and \(\mathcal{C}\) captures the space of discount factors up to the decision time. The state transitions are also defined as:  
\[
S_{t+1} = S_t + C_t R_t, \quad C_{t+1} = \gamma C_t,
\]
where we initialize \( S_0 = 0 \) and \( C_0 = 1 \).

For the outer optimization, Equation \eqref{functionh} suggests that the function \( h \) should be updated based on the return distribution of the initial state. This is where the distributional critic plays a key role: it provides an estimate of this return distribution, which is then used to refine \( h \). Proposition \ref{prop} explains how using the quantile representation of the return distribution modifies the definition of the function \( h \). The detailed proof is provided in \ref{app:functionh}.

\begin{proposition}
\label{prop}
Let \( Z \) be a random variable with quantile representation defined by \( q_i = F^{-1}_{Z}(\hat{\tau}_i) \) for \( i = 1, \dots, N \), where \(\tau_i = i / N\) and \(\hat{\tau}_i = (\tau_{i-1} + \tau_i) / 2\). Then, the function \( h \) in Equation \ref{functionh} can be approximated by the piecewise linear function 
\[
\tilde{h}_{\phi,Z}(z) := \sum_{i=1}^{N} w_i \left(q_i + \frac{1}{\hat{\tau}_i}(z - q_i)^-\right),
\]
where the weights are given by \( w_i = \hat{\tau}_i \left(\phi(\tau_{i-1}) - \phi(\tau_i)\right) \).  
\end{proposition}

  This framework is summarized in Algorithm \ref{alg:acsrm}, and a conceptual diagram is illustrated in Figure \ref{fig:both}. In this algorithm, 
\(
\bar{\mathcal{T}}^{\pi}: \mathscr{P}(\mathbb{R})^{\bar{\mathcal{X}} \times \mathcal{A}}\rightarrow \mathscr{P}(\mathbb{R})^{\bar{\mathcal{X}} \times \mathcal{A}},
\)
represents the distributional Bellman operator for the extended MDP.

\begin{figure*}[t]
    \begin{subfigure}{0.48\textwidth}
        \centering
        \resizebox{\textwidth}{!}{
\begin{tikzpicture}[
    node distance=2.5cm,
    every node/.style={align=center, font=\sffamily\small},
    env/.style={rectangle,     draw=blue!50!black,  fill=blue!10,  minimum width=3.5cm, minimum height=4.0cm, rounded corners=5pt, drop shadow},
    buffer/.style={rectangle,  draw=gray!50!black,  fill=gray!20,  minimum width=2.5cm, minimum height=2.5cm, rounded corners=5pt, drop shadow},
    process/.style={rectangle, draw=green!50!black, fill=green!10, minimum width=2.5cm, minimum height=2.5cm, rounded corners=5pt, drop shadow},
    risk/.style={rectangle,    draw=red!50!black,   fill=red!10,   minimum width=1.5cm, minimum height=1.5cm, rounded corners=5pt, drop shadow},
    arrow/.style={->, thick, -{Stealth}, bend angle=15, color=black!80},
]

\node[env] (env) {
    \includegraphics[width=1cm]{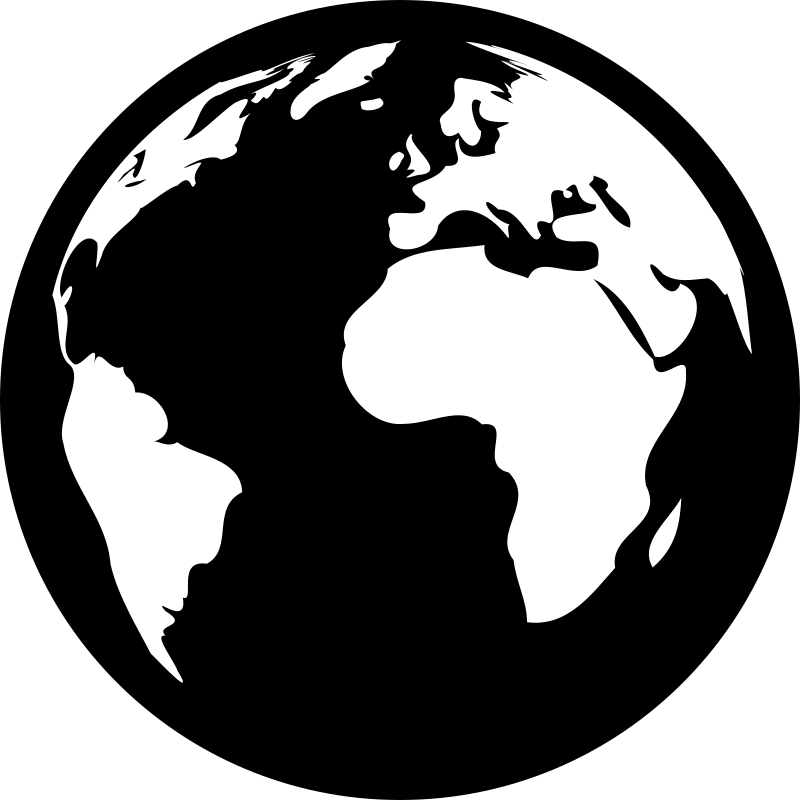} \\[15pt]
    \includegraphics[width=1cm]{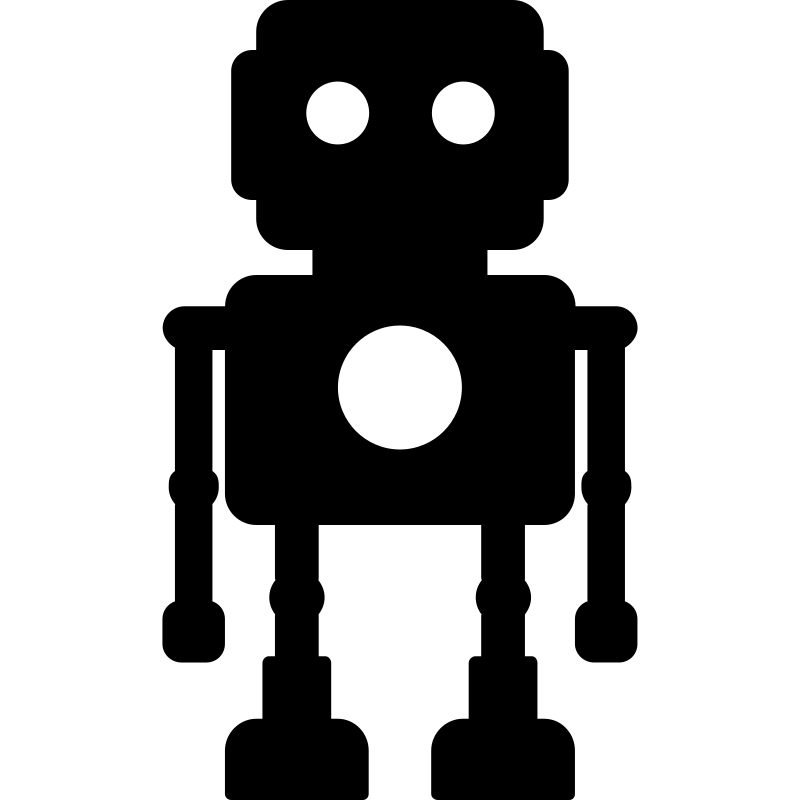}
};

\draw[arrow] ([xshift=-0.6cm, yshift=0.7cm] env.center) to[bend right=50] ([xshift=-0.6cm, yshift=-0.7cm] env.center); 
\draw[arrow] ([xshift=0.6cm, yshift=-0.7cm] env.center) to[bend right=50] ([xshift=0.6cm, yshift=0.7cm] env.center);

\node[buffer, right=of env, xshift=1cm, yshift=2cm] (buffer) {
    \includegraphics[width=1cm]{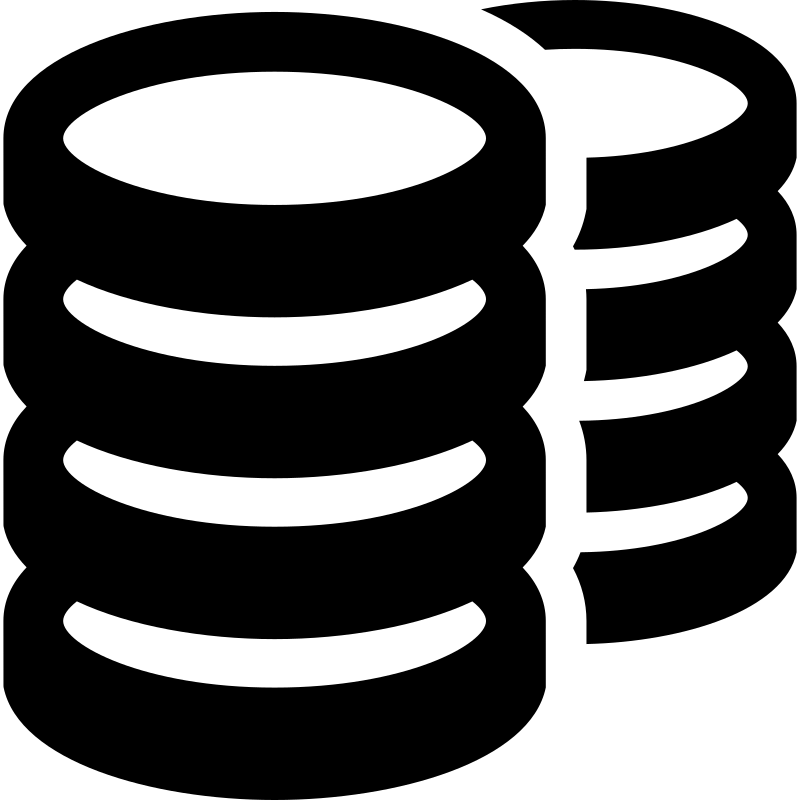} \\[5pt]
    \large Dataset \(\mathcal{D}\)
};

\node[process, below=of buffer, xshift=-3cm] (actor) {
    \includegraphics[width=1cm]{figs/robot.png} \\[5pt]
    \large Actor 
};

\node[process, below=of buffer, xshift=3cm] (critic) {
    \includegraphics[width=1cm]{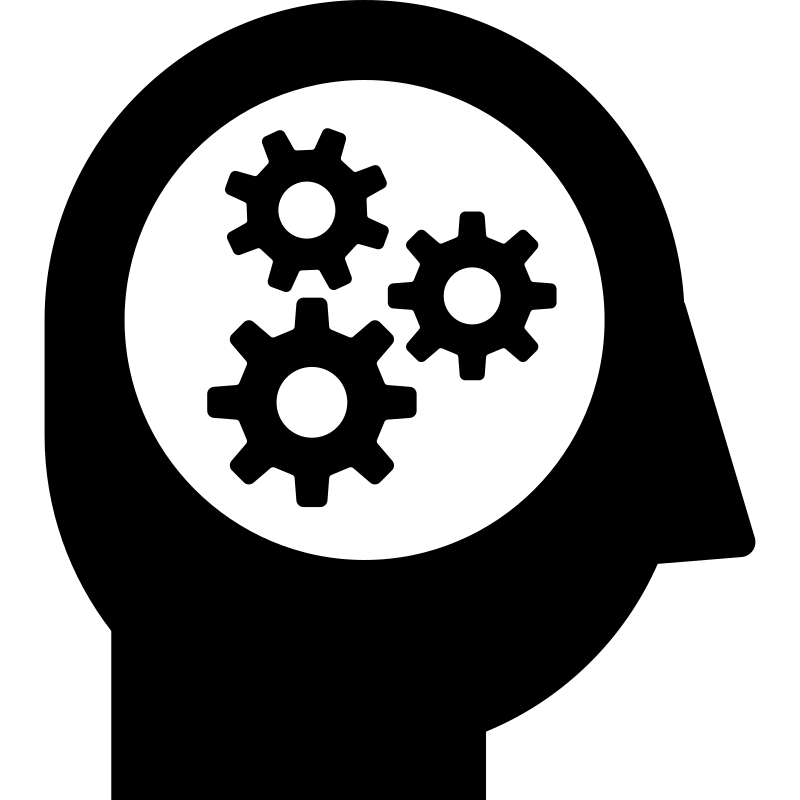} \\[5pt]
    \large Distributional \\ \large Critic
};

\node[risk, below=of buffer, yshift=-2.5cm] (functionh) {
    \includegraphics[width=0.7cm]{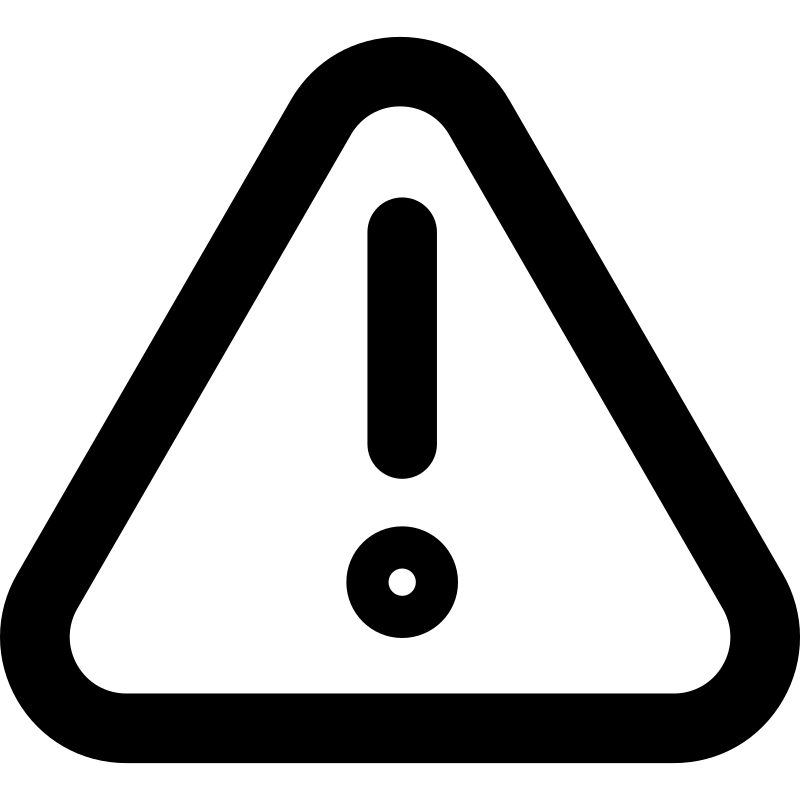} \\[2pt]
    \large Function $h$
};

\node[below=of buffer, yshift=1cm] (mid) {};

\draw[arrow, blue] (env.north) |- ++(0,2) -|  node[pos=0.25, above] {\large \(\bm{\{\bar{x}, a, r, \bar{x}^\prime\}}\)} (buffer.north) ;
\draw[thick] (buffer.south) to node {} (mid.north); 
\draw[arrow] (mid.north) -| node[pos=0.25, above] {\large \(\bm{\{\bar{x}\}}\)} (actor.north);
\draw[arrow] (mid.north) -| node[pos=0.25, above] {\large \(\bm{\{\bar{x}, a, r, \bar{x}^\prime\}}\)} (critic.north);

\draw[arrow] ([yshift=-0.5cm] critic.west) -- node[pos=0.5, above] {\large \(\bm{G^{\pi_{k}}(\bar{x},a)}\)} ([yshift=-0.5cm] actor.east);
\draw[arrow] ([yshift=0.5cm] actor.east)  -- node[pos=0.5, above] {\large \(\bm{\pi_{k}(\bar{x}^\prime)}\)} ([yshift=0.5cm] critic.west);

\draw[arrow, blue] (actor.west) -| node[pos=0.25, below] {\large \(\bm{\pi_{k+1}}\)} (env.south);

\draw[arrow] (critic.south) |- node[pos=0.75, above] {\large \(\bm{G^{\pi_{k}}}\)} (functionh.east);
\draw[arrow] (functionh.west) -| node[pos=0.25, above] {\large \(\bm{h_{\phi, G^{\pi_{k}}}}\)} (actor.south);

\end{tikzpicture}
}
        \label{fig:diagram}
    \end{subfigure}
    \hfill
    \begin{subfigure}{0.48\textwidth}
        \centering
        \includegraphics[width=\textwidth]{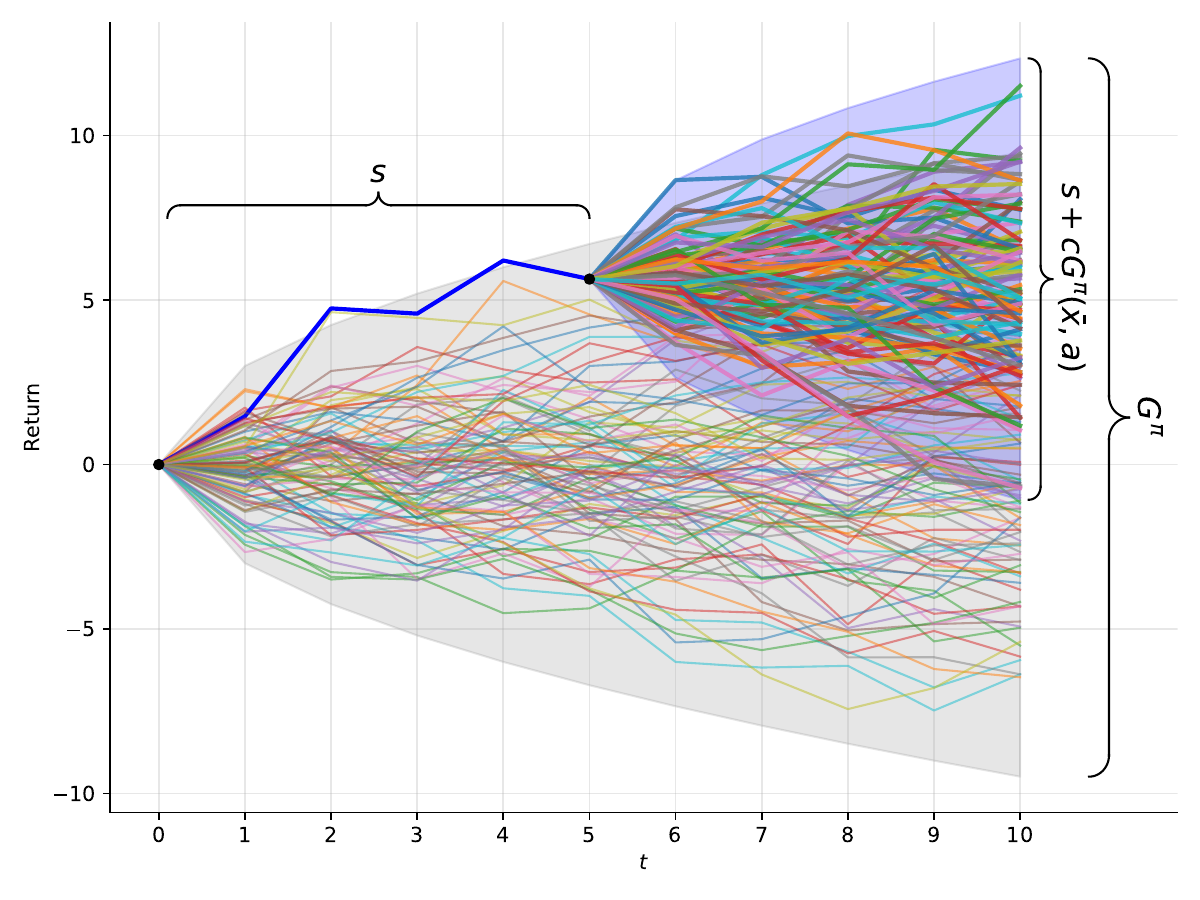}
        \label{fig:scG}
    \end{subfigure}
    \caption{{\textbf{(Left)} Diagram of our framework. The blue connections indicating the interaction of the new policy with the environment and storing the new interaction in the dataset are not present in the offline setting. Finding the solution of the outer optimization can be interpreted as updating the actor's perception of the initial state's return distribution $(G^{\pi_k})$, which is then combined with the actor's risk preference $\phi$ to calculate the Q values. \textbf{(Right)} The static SRM is defined for \( G^\pi \), so calculating the value of a future state-action requires an adjustment of \( G^{\pi}(\bar{x}, a) \) to the initial timestep using \( c \) and \( s \). }}
    \label{fig:both}
\end{figure*}

\begin{algorithm}
\DontPrintSemicolon
\SetAlgoLined
\SetNoFillComment
\LinesNotNumbered 
\caption{Static Spectral Risk Actor-Critic Framework: Bi-level Optimization} 
\label{alg:acsrm}
\textbf{Input:} A random initialization of $G^{\pi_0}$\;
\For{$k = 0, 1, \cdots$}{
\textbf{Step 1:} \tcp{The Closed-form Solution in Equation \ref{functionh}}
$\quad h_{k+1} = \arg\max_h J(\pi_k,h) = \tilde{h}_{\phi, G^{\pi_k}}$ \;
\textbf{Step 2:} \tcp{The Inner Optimization}
$\quad \pi_{k+1} = \arg\max_{\pi} J(\pi,h_{k+1})$ \\
$\quad G^{\pi_{k+1}} = \bar{\mathcal{T}}^{\pi_{k+1}}G^{\pi_{k+1}}$\;
}
\end{algorithm}

In Step 2 of Algorithm \ref{alg:acsrm}, we use a distributional critic to estimate the return distribution for state-action pairs. The corresponding Q-value function for policy gradient is defined as:  
\begin{equation}
\label{eq:q_func}
    Q^{\pi}_{h}(\bar{x}, a):=\mathbb{E}\left[h\left(s + c G^{\pi}(\bar{x}, a)\right)\right]/c,
\end{equation}
where the division by \( c = \gamma^t\) offsets the discounting effect applied to \( G^{\pi}(\bar{x}, a) \). To build intuition for the Q-value function, we examine the structure of \( h \) from Equation \ref{functionh}. The integrand can be rewritten as:  
\begin{align}
    \label{eq:conversion}
    K_\alpha+\frac{1}{\alpha}&\min\left(s + c G^{\pi}(\bar{x}, a),q_\alpha\right) \nonumber \\
    & =  c\left(\bar{K}_\alpha(s,c)+\frac{1}{\alpha}\min\left(G^{\pi}(\bar{x}, a),\dfrac{q_\alpha-s}{c}\right)\right),
\end{align}
where \( q_\alpha = F_{G^\pi}^{-1}(\alpha) \) is the quantile function of \( G^\pi \), and \( K_\alpha := \frac{\alpha q_\alpha - q_\alpha}{\alpha} \) and \( \bar{K}_\alpha(s,c) := \frac{\alpha q_\alpha - q_\alpha + s}{\alpha c} \) are constant terms for a given \( \alpha \) and \( (s, c) \), independent of the action. This formulation highlights that the factor \( c \) can be factored out of \( h \), justifying the division by \( c \) in Equation \eqref{eq:q_func}.

Equation \eqref{eq:conversion} also demonstrates how the extended state variables \( s \) and \( c \) allow the function \( h \) to apply the same risk preference at different time-steps. Comparing the return \( G^\pi \) to returns received at future time-steps requires alignment to a common reference frame. On the left-hand side of the equation, the term \( s + c G^{\pi}(\bar{x}, a) \) adjusts the future return to the initial time-step by incorporating the discount factor \( c \) and past return \( s \). On the right-hand side, the term \( (q_\alpha - s)/c \) represents the adjusted value of \( q_\alpha \) at time \( t \), enabling a direct comparison with \( G^{\pi}(\bar{x}, a) \).

\subsection{Risk-sensitive Policy Learning in the Online Setting: AC-SRM}

We begin with AC-SRM (Actor-Critic with Spectral Risk Measure), designed for the online reinforcement learning setting, where the agent interacts with the environment to collect data during training. AC-SRM directly optimizes a static spectral risk measure (SRM) of the return, allowing the agent to learn policies that reflect specific risk preferences. Using the Q-value function from Equation \ref{eq:q_func}, the state value function is computed as $V^{\pi}_{h}(\bar{x})=\mathbb{E}_{a\sim \pi(a\mid\bar{x}) }\left[Q^{\pi}_{h}(\bar{x},a)\right]$, and the advantage function is defined as $A^{\pi}_{h}(\bar{x}, a):= Q^{\pi}_{h}(\bar{x}, a) - V^{\pi}_{h}(\bar{x})$. With this advantage function, the gradient of the objective with respect to the parameters of policy $\pi_\theta$ can be expressed as follows:
\begin{equation}
\label{eq:policy_gradient}
\nabla_\theta J(\pi_\theta,h) = \mathbb{E}_{d_{\xi_0}^{\pi_\theta}, \pi_\theta}\left[\nabla_\theta\log\pi_\theta(a\mid\bar{x}) A^{\pi_\theta}_{h}(\bar{x}, a)\right]/(1-\gamma),
\end{equation}
where $d_{\xi_0}^{\pi_\theta}$ denote the state occupancy measure when following policy $\pi_\theta$. The following theorem establishes the convergence of this policy gradient method in the finite states and actions setting, with the proof available in \ref{app:ac_inner}. Notably, the assumption of finite extended states implies finite horizon and that the rewards are drawn from a finite set $\hat{\mathcal{R}}\subset[R_{\min},R_{\max}]$.

\begin{theorem}
\label{theorem:ac_inner}
 In the tabular setting (finite states and actions), let the policy $\pi_\theta$ be parameterized by $\theta$ using a softmax function, and let $J(\pi_\theta,h)$ be the performance objective. Assume the parameters $\theta$ are updated according to the Natural Policy Gradient (NPG) scheme \citep{Kakade2001}, based on the policy gradient given in Equation \ref{eq:policy_gradient}. If the updates use a learning rate $\eta_t$ satisfying the Robbins-Monro condition $(\sum_t \eta_t =\infty, \sum_t \eta_t^2 <\infty)$, then $\pi_\theta$ converges to the optimal policy of the inner optimization within the class of representable softmax policies $\Pi$, i.e. $\pi_h^*:=\arg\max_{\pi_\theta \in \Pi}J(\pi_\theta,h)$.
\end{theorem}

While Theorem \ref{theorem:ac_inner} discusses the convergence of the inner optimization, our next theorem discusses the convergence of the overall approach, with the proof available in \ref{app:ac_outer}.

\begin{theorem}
\label{theorem:ac_outer}
The policy updates in Algorithm \ref{alg:acsrm} yield monotonic improvement of the objective $J(\pi)$. When using parameterized policies $\pi_{\theta_k}$, the policy improvement is monotonic as well, meaning $J(\pi_{\theta_{k+1}}) \geq J(\pi_{\theta_k})$. Additionally, the sequence $\{J(\pi_{\theta_k})\}_{k=1,\cdots}$ converges, and $\lim\inf_k \|\nabla_\theta J(\pi_{\theta_k})\| = 0$.
\end{theorem}

Algorithm \ref{alg:stochastic_ac} outlines the steps of our method. To mitigate Q-value overestimation, we employ two critic networks \citep{Hasselt2010}. To further improve training stability, we use a target network, following standard practice in RL, and reduce the frequency of policy updates \citep{Mnih.etal2015, Fujimoto.etal2018}.

\begin{algorithm}[!ht]
\caption{AC-SRM}
\label{alg:stochastic_ac}
\SetAlgoLined
\DontPrintSemicolon
\SetKwInOut{Input}{Input}
\SetKwInOut{Initialize}{Initialize}
{\small
\Input{
Batch size $M$,
Number of quantiles $N$, 
Policy update frequency $d$,
Target smoothing coefficient $\nu$,
Number of policy updates $T_{inner}$,
Number of function $h$ updates $T_{outer}$
}
\Initialize{
Critic networks $G_{\theta_1}, G_{\theta_2}$ and Actor network $\pi_{\theta}=\mathcal{N}(f_\theta,\sigma)$ or $\pi_{\theta}=\operatorname{Categorical}(\operatorname{softmax}(f_\theta))$ with random parameters $\theta_1, \theta_2, \theta$, 
Target network parameters $\theta_1^\prime \leftarrow \theta_1, \theta_2^\prime \leftarrow \theta_2, \theta^\prime \leftarrow \theta$, 
Dataset $\mathcal{D}$
}
\For{$1$ \KwTo $T_{outer}$}{
    \tcp{Update risk function}
    \(
    h =  \tilde{h}_{\phi, G_{\theta_1}(x_0,\pi_{\theta}(x_0))}
    \)
    
    \For{$t=1$ \KwTo $T_{inner}$}{
        \tcp{Collect New Data}
        Observe state $\bar{x}$, select action $a \sim \pi_\theta(\bar{x})$\\
        Execute $a$, observe reward $r$ and next state $\bar{x}^\prime$\\
        Store transition $(\bar{x}, a, r, \bar{x}^\prime)$ into $\mathcal{D}$\\
        \tcp{Update Critic}
        Sample mini-batch of $M$ transitions $(\bar{x}, a, r, \bar{x}^\prime)$ from $\mathcal{D}$\\
        \ForEach{transition in the mini-batch}{
            Sample target action: 
            \[
            a^\prime \sim \pi_{\theta^\prime}(\bar{x}^\prime)
            \]
            
            Compute Q-values for $k = 1, 2$:
            \[
            Q_{k} = \mathbb{E}\left[h\left(s^\prime + c^\prime G_{\theta^\prime_k}(\bar{x}^\prime, a^\prime)\right)\right]/c^\prime
            \]
            
            Select target quantile set:
            \[
            G^\prime(\bar{x}^\prime, a^\prime) = 
            \begin{cases}
            G_{\theta^\prime_1}(\bar{x}^\prime, a^\prime) & \text{if } Q_{1} \leq Q_{2} \\
            G_{\theta^\prime_2}(\bar{x}^\prime, a^\prime) & \text{otherwise}
            \end{cases}
            \]
            
            Compute target quantiles:
            \[
            Y(\bar{x},a) = r + \gamma G^\prime(\bar{x}^\prime, a^\prime)
            \] 
        }
        Minimize quantile regression loss $\mathcal{L}(G_{\theta_k}, Y), k = 1, 2$ (Equation \ref{huberloss}) for the mini-batch
        
        \tcp{Update Actor}
        \If{$t \bmod d = 0$}{
            Compute the advantage function:
            \[
            A(\bar{x},a) = Q_1(\bar{x},a) - \mathbb{E}_{\tilde{a}\sim\pi_{\theta^\prime}}\left[Q_1(\bar{x},\tilde{a})\right]
            \]
            
            Update $\theta$ using the advantage function and the policy gradient in Equation \ref{eq:policy_gradient}

            Update target networks:
            \[
            \theta^\prime_k \leftarrow \nu \theta_k + (1 - \nu) \theta^\prime_k, k=1,2 \quad \theta^\prime \leftarrow \nu \theta + (1 - \nu) \theta^\prime
            \]
        }
    }
}
}
\end{algorithm}

\subsection{Offline Risk-sensitive Learning with Policy Constraints: OAC-SRM}

We next present OAC-SRM (Offline Actor-Critic with Spectral Risk Measure), which extends our framework to the offline reinforcement learning setting, where the agent must learn solely from a fixed dataset $\mathcal{D}$ generated by an unknown behavior policy $\pi_\beta$ without further interaction with the environment. To mitigate distributional shift and reduce extrapolation error, OAC-SRM incorporates a policy constraint that encourages the learned policy to stay close to the behavior policy. 

This framework is especially valuable in multi-stakeholder or high-stakes applications, where risk preferences vary across users or tasks. For example, in financial portfolio management, clients often exhibit different levels of risk aversion. In wealth management or portfolio optimization, historical market data is often the only source available (i.e., the setting is inherently offline). By optimizing policies under different spectral risk measures using the same offline dataset, our framework enables the development of investment strategies tailored to individual risk profiles.

By incorporating this constraint into the policy optimization problem, we arrive at the following objective:
\[
\max_{\pi \in \Pi} \operatorname{SRM}_{\phi} (G^\pi) \quad \text{s.t.} \quad \mathrm{D}_{\mathrm{KL}}(\pi, \pi_\beta) < \epsilon.
\]
As in the previous section, we can use the supremum form of the SRM. Since the constraint does not depend on the function $h$, we can reformulate the problem as an inner-outer optimization:
\[
\max_{h \in \mathcal{H}} \left(\max_{\pi \in \Pi}J(\pi,h) \quad \text{s.t.} \quad \mathrm{D}_{\mathrm{KL}}(\pi, \pi_\beta) < \epsilon \right).
\]

The following theorem closely follows the results of \citet{Peng.etal2019} and \citet{Nair.etal2021}, which derive policy updates under a KL divergence constraint. Our formulation adapts this idea to the risk-sensitive setting by incorporating the risk-adjusted advantage $A_h^{\pi} (\bar{x},a)$, and we treat the Lagrange multiplier $\lambda$ as a tunable hyperparameter. This adaptation enables us to retain the tractability of the update rule while aligning it with the underlying risk-sensitive objective. Our formulation thus generalizes the advantage-weighted actor-critic framework to account for risk preferences in the offline policy learning. The proof of this theorem is available in \ref{app:awac_srm}.
\begin{theorem}
\label{theorem:awac_srm}
For policy search in the inner optimization, the policy constraint can be imposed implicitly (without approximating $\pi_\beta$) using the following policy update rule:
\begin{equation}
    \label{eq:oac_pg}
    \nabla_\theta J(\pi_\theta,h) = \mathbb{E}_{\bar{x},a\sim \mathcal{D}} \left[\nabla_\theta\log \pi_\theta(a \mid \bar{x}) \exp \left(\frac{1}{\lambda}A_h^{\pi_\theta}(\bar{x}, a)\right)\right]/(1-\gamma),
\end{equation}

where the Lagrange multiplier $\lambda$ is treated as a hyperparameter. 
\end{theorem}
The update of the function $h$ in the outer optimization follows the same procedure as in the online setting. Although the distributional critic is less accurate in the offline setting, we empirically show in Section \ref{results} that our algorithm provides an effective approach for learning risk-sensitive policies in offline settings. The detailed algorithm for OAC-SRM is outlined in \ref{app:stochastic_oac}.

\subsection{Risk-sensitive Deterministic Policies: TD3-SRM and TD3BC-SRM}

Finally, we extend our framework to deterministic policies and introduce TD3-SRM and TD3BC-SRM, which build on TD3 \citep{Fujimoto.etal2018} and TD3BC \citep{Fujimoto.Gu2021a} for online and offline settings, respectively. Deterministic policies are well-suited for continuous control tasks and are particularly effective in offline learning, where the benefits of stochastic exploration diminish. To support such settings, we derive a risk-sensitive deterministic policy gradient based on the spectral risk-adjusted value function. Note that, unlike stochastic actors which can be tailored for both discrete and continuous action spaces (e.g., using softmax and Gaussian parameterizations, respectively), deterministic actors are limited to MDPs with continuous action spaces.

Our proposed TD3-SRM retains the core components of TD3, including two distributional critics and a deterministic actor. The critic estimates the return distribution via quantile regression, and the actor is trained to maximize the SRM-adjusted Q-value, defined by applying a spectral risk measure to the estimated quantiles. In the offline setting, we adapt this framework as TD3BC-SRM, where the actor objective is augmented with a behavior cloning loss to keep the learned policy close to the behavior policy.

Building on the definition of the action value function (Equation \ref{eq:q_func}) and the deterministic policy gradient \citep{Silver.etal2014a}, we derive a risk-sensitive deterministic policy gradient, which can be approximated as follows.
\begin{align}
\label{eq:td3srm_pg}
    & \nabla_\theta J(\pi_\theta, h)  \nonumber\\ 
    & \approx \mathbb{E}_{\bar{x}\sim d_{\xi_0}^{\pi_\theta}}\left[\left.\nabla_\theta \pi_\theta(\bar{x}) \nabla_{a} Q^{\pi_\theta}_{h}(\bar{x}, a) \right|_{a=\pi_\theta(\bar{x})}\right]  \nonumber\\
    & = \mathbb{E}_{\bar{x}\sim d_{\xi_0}^{\pi_\theta}}\left[\left.\nabla_\theta \pi_\theta(\bar{x}) \mathbb{E}\left[h^\prime(s+cG^{\pi_\theta}\left(\bar{x}, \pi_\theta(\bar{x})\right))\nabla_{a} G^{\pi_\theta}(\bar{x}, a)\right|_{a=\pi_\theta(\bar{x})}  \right] \right].
\end{align} 
Since $h_{\phi, G^{\pi_\theta}}(z)$ is differentiable almost everywhere with derivative $h_{\phi, G^{\pi_\theta}}^{\prime}(z)=\phi\left(F_{G^{\pi_\theta}}(z)\right)$, we have:
\[
h^\prime(s+cG^{\pi_\theta}\left(\bar{x}, \pi_\theta(\bar{x})\right)) = \phi\left(F_{G^{\pi_\theta}}(s+cG^{\pi_\theta}\left(\bar{x}, \pi_\theta(\bar{x})\right))\right)
\]
Compared to the D4PG algorithm \citep{Barth-Maron.etal2018a}, which employs a deterministic policy gradient with a distributional critic, our approach has an additional \( \phi\left(F_{G^{\pi_\theta}}(s + c G^{\pi_\theta}(\bar{x}, a))\right) \) term for risk sensitivity. This term evaluates the relationship between \( G^{\pi_\theta} \) and \( s + c G^{\pi_\theta}(\bar{x}, a) \), and adjusts the coefficient of \( \nabla_a G^{\pi_\theta}(\bar{x}, a) \) according to the agent's risk preferences, defined by the risk spectrum \( \phi \).

To effectively apply this policy gradient in the offline setting while addressing the challenge of distributional shift, a straightforward yet robust approach is to incorporate a regularization term for behavior cloning:
\begin{align}
\label{eq:td3bcsrm_pg}
    \nabla_\theta J(\pi_\theta, h)  \approx \mathbb{E}_{\bar{x},\hat{a} \sim \mathcal{D}} &\left[\nabla_\theta \pi_\theta(\bar{x}) \nabla_{a} Q^{\pi_\theta}_{h}(\bar{x}, a) \right|_{a=\pi_\theta(\bar{x})} \nonumber\\
    & - \lambda\left.\nabla_\theta \pi_\theta(\bar{x}) \left(\pi_\theta(\bar{x})-\hat{a}\right)\right].
\end{align}
in which $\lambda$ controls the strength of the regularizer. The regularization term encourages the policy to produce actions that align with those observed in the dataset.

Deterministic policies, including their risk-sensitive variants, share common issues such as sensitivity to hyperparameter selection and exploitation of errors in the Q-function. To address these challenges, modifications proven to improve deterministic policy performance can be applied in the risk-sensitive context as well. The Twin Delayed DDPG (TD3) algorithm introduces three key techniques that enhance the stability and performance of standard DDPG: Clipped Double-Q Learning, Delayed Policy Updates, and Target Policy Smoothing. We incorporate these techniques to stabilize training with our risk-sensitive deterministic policy gradients. The detailed steps for TD3-SRM and TD3BC-SRM are provided in \ref{app:algodeterministic}.

\section{Experiment Results}
\label{results}

In this section, we benchmark our algorithm against other risk-neutral and risk-sensitive algorithms in both online and offline settings. As our baselines, we select state-of-the-art model-free off-policy algorithms. In the online setting, we benchmark our algorithm against risk-neutral algorithms such as SAC \citep{Haarnoja.etal2018b} and TD3 \citep{Fujimoto.etal2018}, as well as risk-sensitive algorithms like DSAC \citep{Ma.etal2020}. Similarly, in the offline setting, we use AWAC \citep{Nair.etal2021}, CQL \citep{Kumar.etal2020a}, IQL \citep{Kostrikov.etal2021a}, and TD3+BC \citep{Fujimoto.Gu2021a} for risk-neutral comparisons, and ORAAC \citep{Urpi.etal2021a} and CODAC \citep{Ma.etal2021} for risk-sensitive evaluations. 

Risk-sensitive algorithms are identified by their name and the risk measure they optimize. For example, our actor-critic algorithm using CVaR is referred to as AC-CVaR, while DSAC with the iterative CVaR risk measure is called DSAC-iCVaR. The results in both online and offline settings are normalized so that a score of 0 corresponds to a random policy, and a score of 100 corresponds to an expert policy. The hyperparameters of each model, as well as the values used for normalization, are available in the \ref{details}.

We evaluate these models in environments related to finance, healthcare, and robotics. Descriptions of these environments are provided in each section. In offline RL, the policy used to generate the dataset plays a critical role in shaping the learned policy. Following established conventions in the offline RL literature \citep{Fu.etal2021}, we detail the datasets employed in each of our offline experiments. The \textit{Medium} dataset contains transitions from an early-stopped SAC policy, while the \textit{Medium-Replay} dataset is comprised of the replay buffer collected during the training of the Medium policy. The \textit{Medium-Expert} dataset includes a mixture of sub-optimal and expert-level data, and the \textit{Expert-Replay} dataset consists of the replay buffer from training the SAC policy.

\subsection{Mean-reverting Trading.}
In this environment, the agent aims to make a profit by trading an asset that follows a mean-reverting Ornstein-Uhlenbeck process, described by the equation 
\[
\mathrm{d} P_t = \kappa(\zeta - P_t) \mathrm{d} t + \sigma \mathrm{d} W_t,
\]
where $\zeta=1.0$ is the long-term mean, $\kappa=2.0$ controls the speed of reversion to the mean, $\sigma=1.0$ is the volatility of the random fluctuations, and $W_t$ is a standard Wiener process. At each time step, $t = 0, \dots, T-1$, the agent takes an action $a_t \in (-a_{\max}, a_{\max})$, which adjusts its inventory $q_t \in (-q_{\max}, q_{\max})$, representing the quantity of the asset traded. 

The state is represented by a 3-dimensional vector comprising the asset price, the quantity of assets the agent possesses, and the remaining time. The rewards are also structured as follows: for $t = 0$ to $T-2$, the agent’s reward is $r_t = -a_t P_t - \varphi a_t^2$, incorporating transaction costs ($\varphi=0.005$). At the final time step, $t = T-1$, the reward is modified to include an additional term penalizing the agent for holding any remaining inventory, given by $r_{T-1} = -a_{T-1} P_{T-1} - \varphi a_{T-1}^2 + q_T P_T - \psi q_T^2$, where $\psi=0.5$ represents the terminal penalty.

First, we compare our online algorithms, TD3-CVaR and AC-CVaR, to DSAC-iCVaR. All three algorithms are optimized using the CVaR objective with \(\alpha = 0.2\). To evaluate their performance, we simulate 10,000 trajectories and compute the \(\operatorname{CVaR}_\alpha\) of their returns. The x-axis in Figure \ref{fig:trading1} represents CVaR levels, ranging from 0.1 to 1.0, while the y-axis shows the average normalized score \(\pm\) standard deviation across five random seeds. The results indicate that while TD3-CVaR and DSAC-iCVaR achieve similar \(\operatorname{CVaR}_{0.2}\) values, TD3-CVaR performs better in terms of expected return.  

In the offline setting, we observe similar trends. These experiments are conducted using the Expert-Replay and Random datasets. As shown in Figure \ref{fig:trading2}, TD3BC-CVaR and CODAC-iCVaR achieve comparable \(\operatorname{CVaR}_{0.2}\), but differences emerge in their expected returns. When trained on the Random dataset, all models experience a decline in performance due to the lack of expert demonstrations, as expected. However, as illustrated in Figure \ref{fig:trading3}, both of our models successfully extract useful policies from the dataset and outperform ORAAC-iCVaR and CODAC-iCVaR.

\begin{figure*}[!ht]
     \centering
     \begin{subfigure}[b]{0.32\textwidth}
         \centering
         \includegraphics[width=\textwidth]{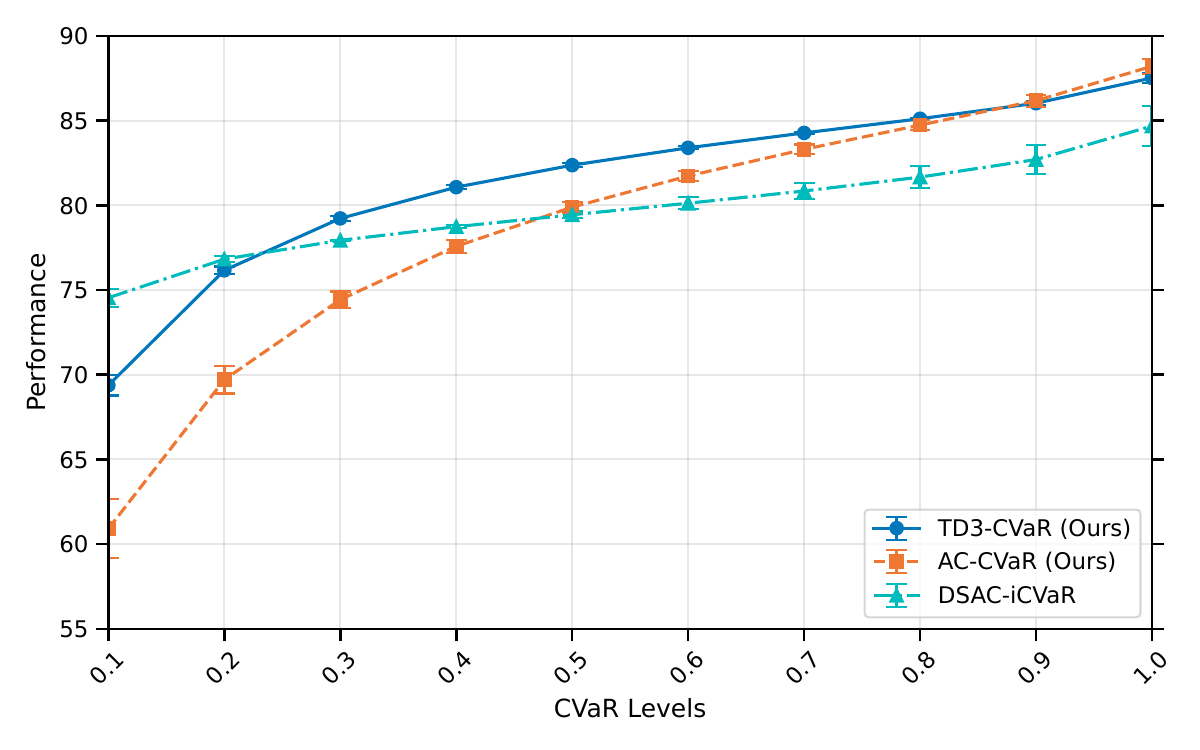}
         \caption{Online}
         \label{fig:trading1}
     \end{subfigure}
     \hfill
     \begin{subfigure}[b]{0.32\textwidth}
         \centering
         \includegraphics[width=\textwidth]{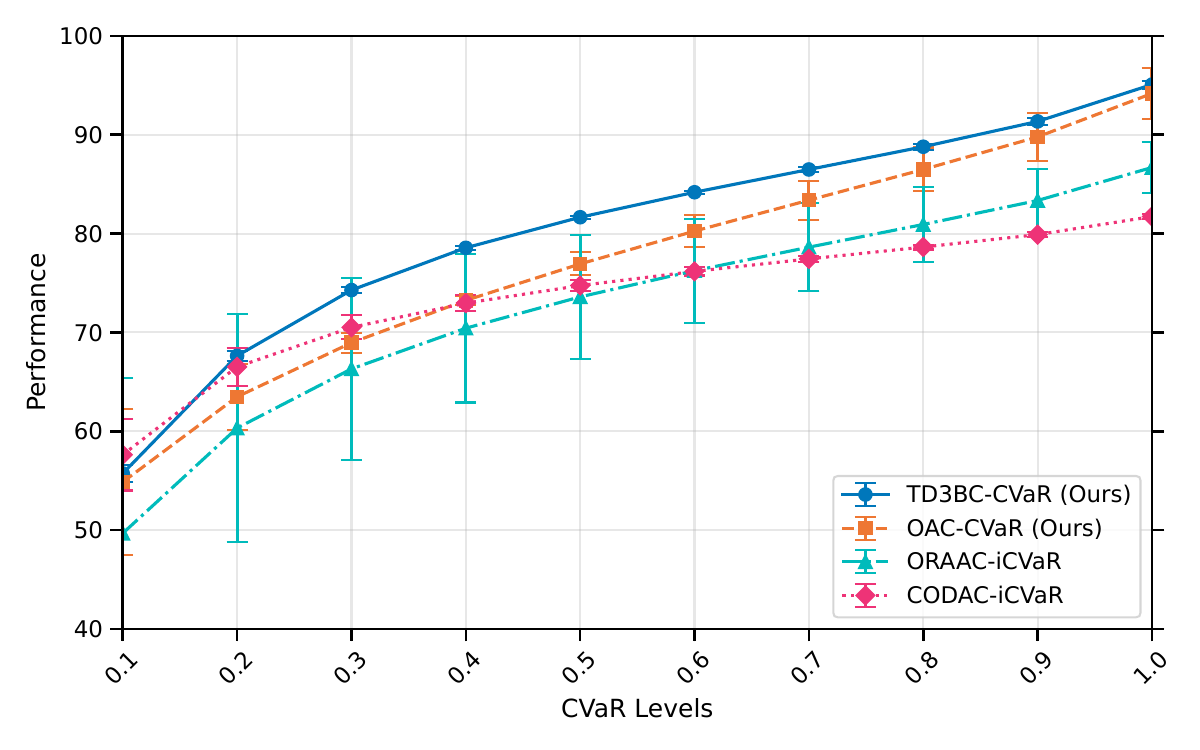}
         \caption{Offline (Expert)}
         \label{fig:trading2}
     \end{subfigure}
     \hfill
     \begin{subfigure}[b]{0.32\textwidth}
         \centering
         \includegraphics[width=\textwidth]{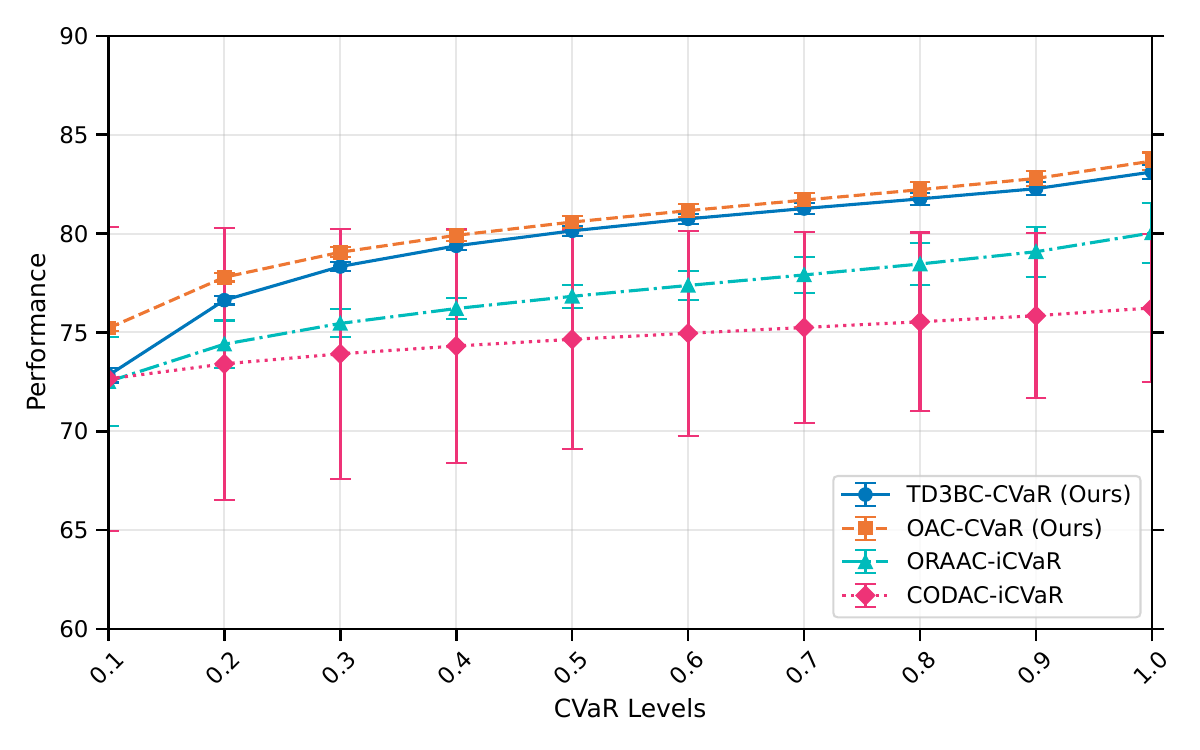}
         \caption{Offline (Random)}
         \label{fig:trading3}
     \end{subfigure}
        \caption{Risk-sensitive performance comparison of online and offline RL algorithms with the $\operatorname{CVaR}_{0.2}$ objective.}
        \label{fig:trading_plot}
\end{figure*}

To isolate the impact of the risk measure from other design choices, we conduct an additional experiment comparing our TD3-Exp model to its iterative risk measure counterpart, TD3-iExp. Figure \ref{fig:exp_online} highlights that iterative risk measures lead to more conservative policies. While both risk measures produce policies that align with their respective risk-sensitive objectives, policies optimized with iterative risk measures yield lower expected returns at higher risk levels.

\begin{figure*}[!ht]
\centering
\includegraphics[width=0.95\linewidth]{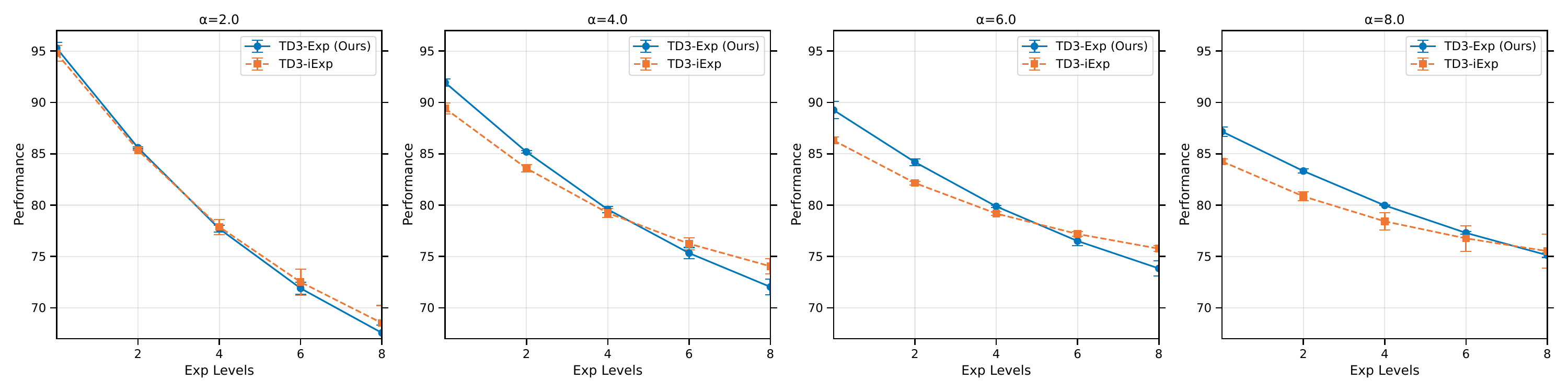}
\caption{Risk-sensitive performance comparison of TD3-Exp and TD3-iExp.}
\label{fig:exp_online}
\end{figure*}

\subsection{Portfolio Allocation}
To evaluate our algorithms in a more realistic setting, we design a portfolio allocation environment based on real-world market data. While RL has proven to be a suitable and widely used approach for portfolio management \citep{Almahdi.Yang2017, Pendharkar.Cusatis2018, Park.etal2020, Wang.Ku2022}, the problem has received little attention in the offline RL setting, despite its practical relevance in financial applications where interaction with the market is costly. Our primary goal with this environment is to demonstrate the effectiveness of learning various risk-sensitive policies in an offline setting. The environment is based on historical daily price data obtained from Yahoo Finance for a set of core assets: SPY (an ETF tracking the S\&P 500 index, representing U.S. equities), GLD (an ETF tracking the price of gold, representing a key alternative asset), and a risk-free cash component with 0\% return. This selection provides a balance between equity market exposure, a non-correlated safe-haven asset, and a risk-free option, capturing essential trade-offs in portfolio allocation among growth, diversification, and capital preservation. The historical data spans twenty years, from January 1, 2005, to December 31, 2024, covering a range of market regimes including bull markets, recessions, and periods of financial stress. To enable realistic backtesting and prevent look-ahead bias, the data is chronologically split: the initial 80\% of the time series serves as the training dataset for policy learning, while the remaining 20\% forms a held-out test set for evaluating the generalization performance of trained agents on unseen future data.

Within the environment, we use historical daily closing prices \( p_{i,t} \) for each asset \( i \) at time \( t \) to compute log-returns defined as \( r_{i,t} = \log(p_{i,t} / p_{i,t-1}) \). The agent observes a state composed of the most recent \( m = 5\) log-returns for each asset, along with the current portfolio weights. A TD3 algorithm is trained to learn an expert policy that outputs the portfolio allocation weights for the next trading period. To reflect practical trading constraints, we enforce that the policy produces non-negative weights that sum to one, effectively prohibiting short selling. The reward at each step is the log-return of the portfolio value, computed using the new weights and asset returns. We also deduct a transaction cost equal to 0.25\% of the trading value at each allocation step to simulate realistic market frictions.

The replay buffer accumulated during TD3 training serves as the Expert-Replay dataset, which we subsequently use to train alternative policies tailored to different risk preferences. This approach enables us to leverage a single dataset to learn a diverse set of risk-sensitive policies without requiring further interaction with the environment.

\begin{figure*}[!ht]
     \centering
     \begin{subfigure}[b]{0.32\textwidth}
         \centering
         \includegraphics[width=\textwidth]{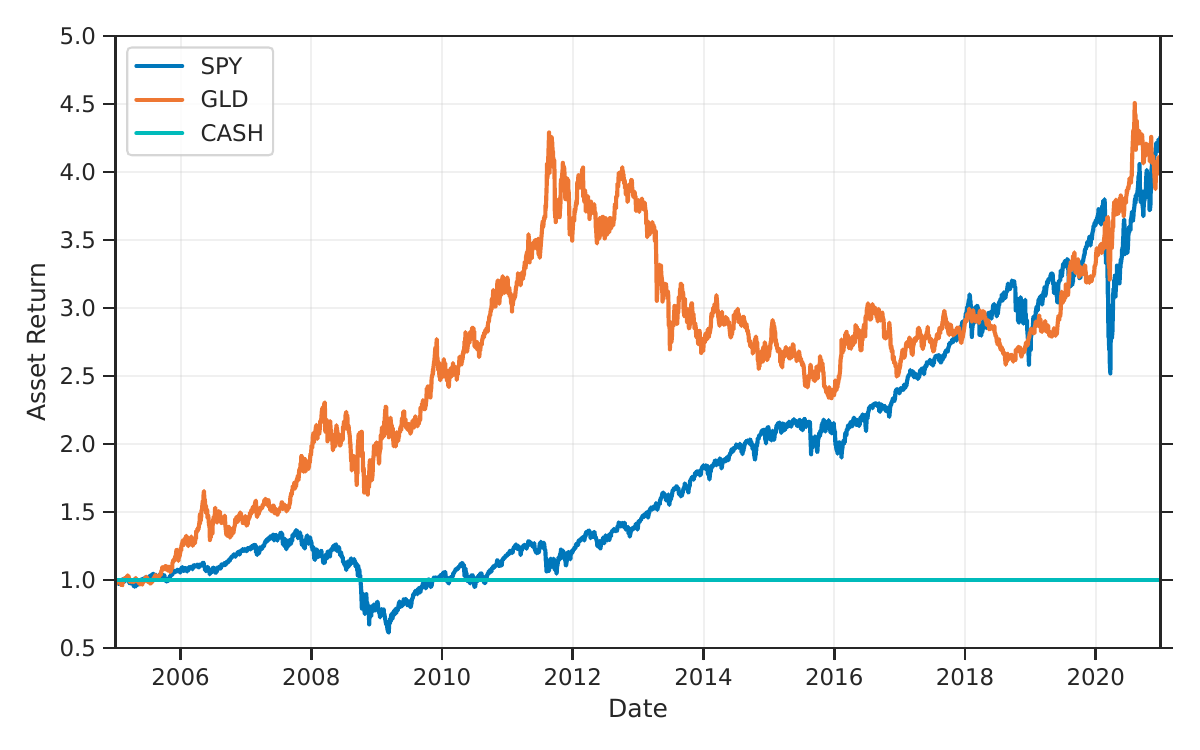}
         \caption{Training Period}
         \label{fig:train}
     \end{subfigure}
     \hfill
     \begin{subfigure}[b]{0.32\textwidth}
         \centering
         \includegraphics[width=\textwidth]{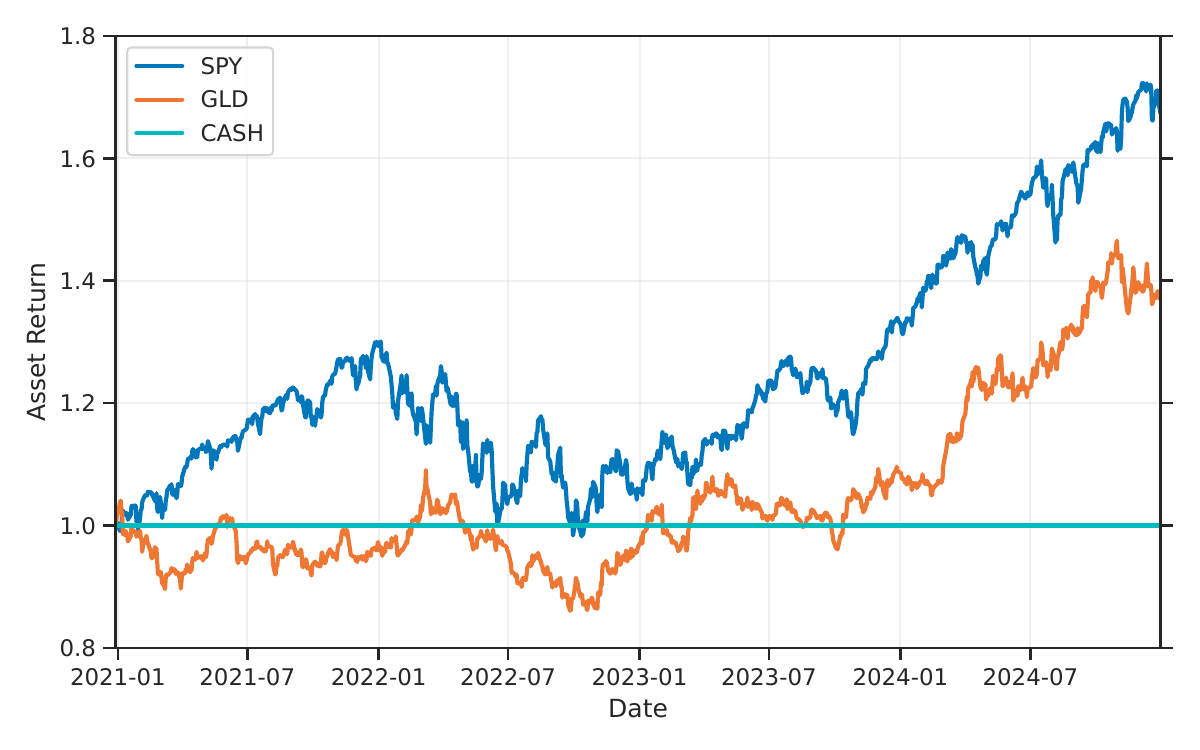}
         \caption{Testing Period}
         \label{fig:test}
     \end{subfigure}
     \hfill
     \begin{subfigure}[b]{0.32\textwidth}
         \centering
         \includegraphics[width=\textwidth]{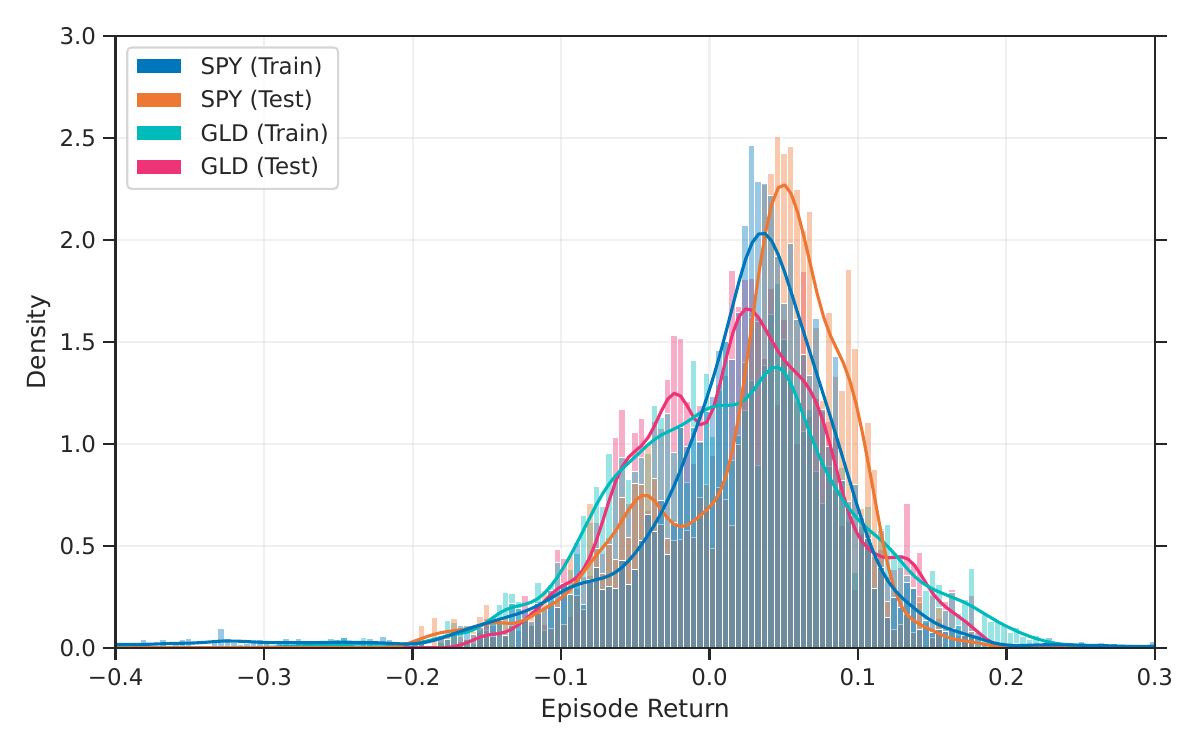}
         \caption{Distribution of Returns}
         \label{fig:dist}
     \end{subfigure}
        \caption{Figures \ref{fig:train} and \ref{fig:test} display the asset returns during the training and testing period. Figure \ref{fig:dist} illustrates the distribution of asset returns in an episode.}
        \label{fig:data_portfo}
\end{figure*}

Each episode in the environment spans 63 trading days, corresponding to approximately three calendar months. The starting point of each episode is randomly selected from the historical data to provide diverse market conditions. Since rewards are defined as log-returns, the cumulative reward over an episode corresponds to the log-return of the portfolio over the entire episode, offering a coherent measure of long-term performance. Figure \ref{fig:data_portfo} provides an overview of the asset data used in the environment. Figures \ref{fig:train} and \ref{fig:test} show the asset return trajectories for SPY, GLD, and the risk-free asset (CASH) during the training and testing periods, respectively. Figure \ref{fig:dist} illustrates the distribution of episode returns across assets, highlighting the variability and skewness in return profiles over different market regimes.

\begin{figure*}[!ht]
     \centering
     \begin{subfigure}[b]{0.48\textwidth}
         \centering
         \includegraphics[width=\textwidth]{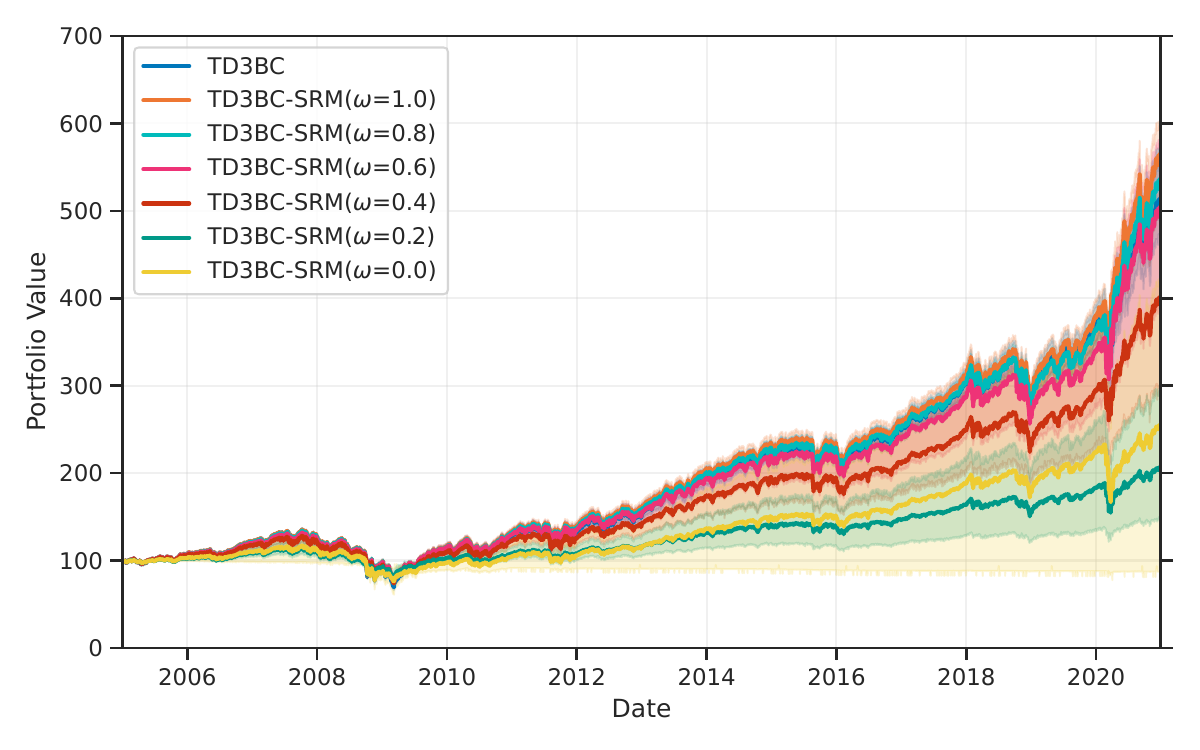}
         \caption{Training Period}
         \label{fig:portfolio1}
     \end{subfigure}
     \hfill
     \begin{subfigure}[b]{0.48\textwidth}
         \centering
         \includegraphics[width=\textwidth]{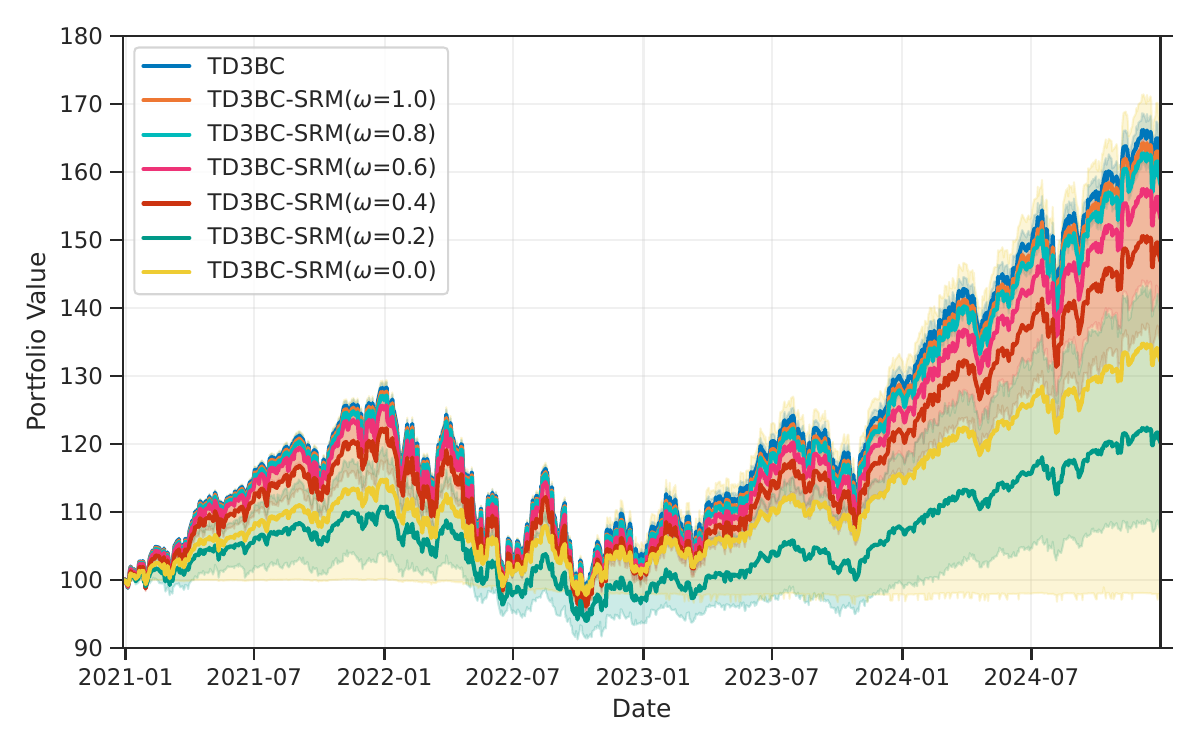}
         \caption{Testing Period}
         \label{fig:portfolio2}
     \end{subfigure}
        \caption{The value of the portfolio in the training and testing period for different policies. The shaded regions correspond to 95\% confidence interval.}
        \label{fig:portfolio}
\end{figure*}

Using the Expert-Replay dataset, we train TD3BC and TD3BC-MC with $\alpha = 0.2$ and $\omega \in \{0.0, 0.2, 0.4, 0.6, 0.8, 1.0\}$, where $\omega$ controls the weight of the expected return in the Mean-CVaR objective. This setup allows us to capture a diverse set of policies aligned with different levels of risk sensitivity. We evaluate the performance of these algorithms during both the training and testing periods using four different metrics. To facilitate fair comparison across algorithms, we first normalize each algorithm’s performance relative to the expert TD3 policy on the training (2.81\% average return) and testing (2.48\% average return) period, following the normalization procedure discussed earlier. We then compute and report the Mean and Conditional Value-at-Risk (CVaR) at a confidence level $\alpha = 0.2$ for this normalized performance, capturing both the average and the worst-case behavior of each algorithm.

In addition, we report the Sharpe Ratio (SR), defined as:
\[
\text{SR} = \frac{\mathbb{E}[R]}{\sigma(R)},
\]
where $R$ represents the daily log-returns over the testing period, $\mathbb{E}[R]$ is the mean return, and $\sigma(R)$ is its standard deviation. A higher Sharpe Ratio indicates better risk-adjusted performance. We also include the Maximum Drawdown (MDD), which measures the largest loss from a peak to a subsequent trough in portfolio value over time. Formally, for a time series of portfolio values $\{V_t\}_{t=1}^T$, MDD is defined as:
\[
\text{MDD} = \max_{t \in [1, T]} \left( \max_{s \in [1, t]} \frac{V_s - V_t}{V_s} \right).
\]

Together, these four metrics provide a comprehensive evaluation of each algorithm’s return profile and risk characteristics. Tables \ref{tab:portfolio_train} and \ref{tab:portfolio} summarize the results for all algorithms in the training and testing periods, respectively.

Unlike TD3BC, which lacks the flexibility to adapt to different risk preferences, our algorithm is explicitly designed to address worst-case outcomes and produce policies tailored to varying levels of risk aversion. As seen in Table \ref{tab:portfolio_train}, during the training period, we observe that the worst-case performance improves as more weight is allocated to CVaR in the Mean-CVaR objective, an expected result given that training and evaluation occur on the same data. Notably, the Mean-CVaR objective with $\omega = 0.2$ achieves the best $\operatorname{CVaR}_{0.2}$ and MDD performance. This outcome can be attributed to more stable training under the Mean-CVaR objective compared to directly optimizing CVaR, as reflected by lower standard deviations across different random seeds.

\begin{table}[!ht]
	\caption{Results for the offline algorithms evaluated in the Training period of the Portfolio Allocation environment. The $\pm$ symbol represents the standard deviation across seeds.}
	\label{tab:portfolio_train}
	\centering
	\resizebox{\linewidth}{!}{
	\begin{tabular}{llllllll}
		\toprule Metric             & TD3BC                & \shortstack{TD3BC-MC\\($\omega$=1.0)} & \shortstack{TD3BC-MC\\($\omega$=0.8)} & \shortstack{TD3BC-MC\\($\omega$=0.6)} & \shortstack{TD3BC-MC\\($\omega$=0.4)} & \shortstack{TD3BC-MC\\($\omega$=0.2)} & \shortstack{TD3BC-MC\\($\omega$=0.0)} \\
		\midrule $\mathbb{E}$       & 92.56$\pm$0.88         & \textbf{93.88$\pm$1.39}                 & 92.49$\pm$2.43                          & 91.55$\pm$4.68                          & 88.71$\pm$4.61                          & 80.26$\pm$5.35                          & 80.79$\pm$11.33                         \\
		$\operatorname{CVaR}_{0.2}$ & 1.20$\pm$5.79          & 6.93$\pm$3.13                           & 8.30$\pm$3.17                           & 11.51$\pm$8.49                          & 14.38$\pm$15.14                         & \textbf{28.61$\pm$24.38}                & 22.66$\pm$48.01                         \\
		SR                          & 0.65$\pm$0.04 &        \textbf{0.69$\pm$0.03}                    & 0.68$\pm$0.03                           & 0.66$\pm$0.08                           & 0.60$\pm$0.07                           & 0.38$\pm$0.12                           & 0.20$\pm$0.41                           \\
		MDD                         & -48.77$\pm$2.83        & -45.45$\pm$1.66                         & -44.88$\pm$2.15                         & -43.19$\pm$5.78                         & -41.72$\pm$9.70                         & \textbf{-31.12$\pm$12.37}               & -36.21$\pm$25.66                        \\
		\bottomrule
	\end{tabular}
	}
\end{table}

Turning to the results of the policies on the unseen testing period (Table \ref{tab:portfolio}), we observe similar improvements in the worst-case performance of TD3BC-MC, especially compared to the risk-neutral TD3BC. In this period, TD3BC-MC with $\omega = 0.0$ achieves the highest $\operatorname{CVaR}_{0.2}$ and Sharpe Ratio, demonstrating the effectiveness of our risk-sensitive algorithm. An interesting observation is that, in Table \ref{tab:portfolio_train} for the training period, we see a significant increase in $\operatorname{CVaR}_{0.2}$ and a decline in average performance as $\omega$ decreases from 0.4 to 0.2. Table \ref{tab:portfolio} shows that this pattern is also present in the testing period. This finding has important practical implications as different risk-sensitive policies trained offline on the dataset exhibit consistent behavior when evaluated on unseen data. Consequently, users can train a diverse set of risk-sensitive policies without additional environment interaction and select among them based on their desired risk-return profile.

\begin{table}[!ht]
	\caption{Results for the offline algorithms evaluated in the Testing period of the Portfolio Allocation environment. The $\pm$ symbol represents the standard deviation across seeds.}
	\label{tab:portfolio}
	\centering
	\resizebox{\linewidth}{!}{
	\begin{tabular}{llllllll}
		\toprule Metric & TD3BC & \shortstack{TD3BC-MC\\($\omega$=1.0)} & \shortstack{TD3BC-MC\\($\omega$=0.8)} & \shortstack{TD3BC-MC\\($\omega$=0.6)} & \shortstack{TD3BC-MC\\($\omega$=0.4)} & \shortstack{TD3BC-MC\\($\omega$=0.2)} & \shortstack{TD3BC-MC\\($\omega$=0.0)} \\
		\midrule $\mathbb{E}$           & \textbf{99.97$\pm$1.85} & 99.17$\pm$0.68                          & 97.71$\pm$2.32                          & 96.26$\pm$6.10                          & 94.25$\pm$5.66                          & 84.54$\pm$8.75                          & 87.43$\pm$16.53                         \\
		$\operatorname{CVaR}_{0.2}$     & 0.80$\pm$2.19           & 0.91$\pm$1.71                           & 3.56$\pm$5.37                           & 5.81$\pm$8.89                           & 11.69$\pm$15.52                         & 30.65$\pm$15.93                         & \textbf{33.23$\pm$40.60}                \\
		SR                              & 0.82$\pm$0.01           & 0.81$\pm$0.02                           & 0.81$\pm$0.01                           & 0.76$\pm$0.13                           & 0.76$\pm$0.07                           & 0.46$\pm$0.26                           & \textbf{0.86$\pm$0.01}                  \\
		MDD                             & -24.28$\pm$0.54         & -24.23$\pm$0.46                         & -23.89$\pm$0.47                         & -23.36$\pm$1.64                         & -21.23$\pm$4.62                         & \textbf{-15.29$\pm$4.71}                & -24.47$\pm$0.21                         \\
		\bottomrule
	\end{tabular}
	}
\end{table}

\subsection{HIV Treatment.}
While our previous experiments focused on financial applications, our method is not limited to this domain. To demonstrate its broader applicability, we evaluate it in the context of healthcare using the HIV treatment simulator. The HIV treatment simulator, originally introduced by \citet{Ernst.etal2006} and implemented following the work of \citet{Geramifard.etal2015}, provides a simplified yet widely adopted model of HIV patient treatment in reinforcement learning research \citep{Keramati.etal2020, Rigter.etal2023}. The environment features a six-dimensional state space representing concentrations of various cells and viral loads in the patient’s bloodstream. The agent controls the dosages of two drugs, administered either individually or in combination, and receives rewards based on the patient’s health outcomes. Due to stepwise variability in drug efficacy, the environment is inherently stochastic. Each treatment is applied for 20 consecutive days, with a total of 50 time steps per episode, resulting in an episode length of 1,000 days.

Compared to the trading environment, the HIV treatment environment exhibits substantially greater stochasticity and carries higher real-world stakes, where minimizing adverse outcomes is essential. Accordingly, we prioritize risk aversion in this setting by employing the Mean-CVaR objective with $\alpha=0.1$ and $\omega=0.4$. As shown in Table \ref{tab:hiv_online}, while the iterative Mean-CVaR objective leads to modest improvements in the risk-sensitive performance of DSAC, our algorithm optimized for the static Mean-CVaR objective achieves substantial gains in both expected return and $\operatorname{CVaR}_{0.1}$ when applied to TD3 and AC. Notably, TD3-MC achieves the highest performance with respect to both metrics.

\begin{table*}[!ht]
	\caption{Normalized results for the online experiments in the HIV Treatment environment. We report both the expected return and $\operatorname{CVaR}_{\alpha}$, averaged across 5 different seeds. The $\pm$ symbol represents the standard deviation across seeds.}
	\label{tab:hiv_online}
	\centering
	\resizebox{0.8\linewidth}{!}{
		\begin{tabular}{lllllll}
			\toprule 
            Metric                       & AC-MC (Ours)                 & TD3-MC (Ours)                    & DSAC-iMC                 & AC                 & TD3                      & DSAC         \\ 
			\midrule 
            $\mathbb{E}$                 &               54.99$\pm$12.22 &              \textbf{100.33$\pm$19.43} &                87.92$\pm$13.89 &               48.13$\pm$13.18 &                71.65$\pm$5.61 &                81.15$\pm$16.67 \\
            $\operatorname{CVaR}_{0.1}$  &               16.24$\pm$10.63 &              \textbf{53.80$\pm$10.19}  &                30.79$\pm$8.99 &                10.76$\pm$8.08  &                44.59$\pm$5.37 &                27.65$\pm$10.57 \\
			\bottomrule
		\end{tabular}
	}
\end{table*}

In the offline setting, we use the Medium-Replay dataset, which consists of the replay buffer collected during SAC training, up to the point where the policy reaches 40\% of the expert policy’s performance \citep{Rigter.etal2023}. Using this dataset, our static Mean-CVaR algorithms, OAC-MC and TD3BC-MC, again demonstrate improvements in both expected return and $\operatorname{CVaR}_{0.1}$ over risk-neutral baselines such as AWAC and TD3BC. Among all models evaluated, TD3BC-MC attains the best performance. Conversely, CODAC-iMC which is a distributional variant of CQL with an iterative risk measure underperforms, suggesting that iterative risk measures may be less effective in this context. One possible explanation is that the combined conservativeness of CQL and iterative risk measures prevent the policy from reaching higher-reward states and ultimately reduce both expected return and worst-case performance.

\begin{table*}[!ht]
	\caption{Normalized results for the offline experiments in the HIV Treatment environment. We report the expected return and $\operatorname{CVaR}_{\alpha}$, averaged across 5 seeds. The $\pm$ symbol represents the standard deviation across seeds.}
	\label{tab:hiv_offline}
	\centering
	\resizebox{\linewidth}{!}{
	\begin{tabular}{lllllllll}
        \toprule 
		Metric                       & OAC-MC (Ours)                   & TD3BC-MC (Ours)                &  ORAAC-iMC          &  CODAC-iMC        & CQL                       & IQL                      & AWAC                       & TD3BC    \\          
        \midrule 
		$\mathbb{E}$                 & 52.81$\pm$0.61            & \textbf{76.01$\pm$0.83}  &  35.01$\pm$12.35    &  49.05$\pm$3.97   & 68.61$\pm$1.56            & 58.40$\pm$0.44           & 46.65$\pm$2.23             & 71.82$\pm$1.21   \\
		$\operatorname{CVaR}_{0.1}$  & 10.28$\pm$1.65            & \textbf{43.73$\pm$1.66}  &  1.89$\pm$0.58     &  10.94$\pm$0.96    & 35.11$\pm$2.44           &  19.59$\pm$1.49          & 3.47$\pm$0.63              & 37.33$\pm$2.43    \\
	\bottomrule
	\end{tabular}
	}
\end{table*}

\subsection{Stochastic MuJoCo.} 

To further showcase the generality of our method, we evaluate it in a robotics setting using MuJoCo continuous control tasks. MuJoCo \citep{Todorov.etal2012} is a physics engine for simulating complex physical interactions, widely used in robotics research. These robotics simulators are significantly more complex than previous environments due to their larger state and action spaces. In these environments, transitions are deterministic, with stochasticity only arising from the initial state. To introduce additional stochasticity, we use a stochastic version of MuJoCo, where small perturbations are added to the rewards \citep{Urpi.etal2021a, Ma.etal2021}. 

For HalfCheetah, we modify the reward function as follows:
\[
R_t(s, a) = \bar{r}_t(s, a) - 70 \, \mathbb{I}_{v > \bar{v}} \cdot \mathcal{B}_{0.1},
\]
where $\bar{r}_t(s, a)$ is the original environment reward, $v$ denotes the forward velocity, and $\mathcal{B}_{0.1}$ is a Bernoulli random variable with success probability $p = 0.1$. We set the velocity threshold $\bar{v}$ to 4 for the medium dataset. The maximum episode length for HalfCheetah is set to 200 steps.

For Walker2D and Hopper, the reward function is modified to:
\[
R_t(s, a) = \bar{r}_t(s, a) - p \, \mathbb{I}_{|\theta| > \bar{\theta}} \cdot \mathcal{B}_{0.1},
\]
where $\bar{r}_t(s, a)$ is the original reward, $\theta$ represents the pitch angle, and $\mathcal{B}_{0.1}$ introduces stochastic penalties. The penalty threshold $\bar{\theta}$ is set to 0.5 for Walker2D and 0.1 for Hopper, with corresponding penalty magnitudes $p = 30$ and $p = 50$, respectively. Additionally, if $|\theta|$ exceeds $2\bar{\theta}$, the agent is considered to have fallen, terminating the episode. This stochastic penalty mechanism penalizes unstable postures, encouraging the agent to maintain balance. Both Walker2D and Hopper are evaluated over a maximum of 500 time steps.

These environments are the most complex in our experiments, with added stochasticity making them even more challenging. We use the Mean-CVaR objective with \(\alpha=0.2\) and \(\omega=0.2\) for risk-sensitive algorithms in these settings. In the HalfCheetah and Walker2D environments, both AC-MC and TD3-MC improve performance over their risk-neutral counterparts, AC and TD3, demonstrating that incorporating risk sensitivity can help discover more effective policies. Notably, TD3-MC achieves the best results across all tested algorithms. Moreover, AC-MC and TD3-MC surpass DSAC-iMC in these environments, further highlighting the advantages of our approach over iterative risk measures. In the Hopper environment, AC, TD3, AC-MC, and TD3-MC perform similarly, with little difference in their results. A notable finding is that DSAC-iMC performs worse than DSAC in the Hopper and Walker2D environments. This decline may be due to the use of an iterative risk measure in environments with long horizons which results in overly conservative policies.

\begin{table*}[!ht]
	\caption{Normalized results for the online experiments. We report the expected return and $\operatorname{CVaR}_{\alpha}$, averaged across 5 seeds. The $\pm$ symbol represents the standard deviation across seeds.}
	\label{tab:mujoco_online}
	\centering
	\resizebox{0.8\linewidth}{!}{
		\begin{tabular}{llllllll}
			\toprule 
			Environment                   & Metric                      & AC-MC (Ours)          & TD3-MC (Ours)                   & DSAC-iMC        & AC                       & TD3             & DSAC            \\             
			\midrule 
			\multirow{2}{*}{HalfCheetah}  & $\mathbb{E}$                &                   91.59$\pm$0.96 &                   \textbf{97.68$\pm$1.41} &                 78.57$\pm$9.73 &  66.85$\pm$16.13 &  82.15$\pm$3.75 &        80.86$\pm$1.52 \\
			                               & $\operatorname{CVaR}_{0.2}$ &                   89.14$\pm$1.25 &                   \textbf{93.30$\pm$2.69} &                 69.91$\pm$9.97 &  40.37$\pm$24.17 &  76.35$\pm$4.87 &        67.58$\pm$1.54 \\
			\midrule 
			\multirow{2}{*}{Hopper}       & $\mathbb{E}$                &                  98.73$\pm$17.38 &                  101.20$\pm$4.30 &                54.43$\pm$12.48 &  \textbf{105.21$\pm$2.75} &  100.14$\pm$3.70 &       62.24$\pm$18.44 \\
			                               & $\operatorname{CVaR}_{0.2}$ &                  96.77$\pm$18.59 &                   99.11$\pm$6.21 &                46.58$\pm$13.91 &  \textbf{104.62$\pm$2.88} &   95.30$\pm$5.75 &       52.85$\pm$24.36 \\          
			\midrule 
			\multirow{2}{*}{Walker2D}     & $\mathbb{E}$                &                   90.36$\pm$5.33 &                  \textbf{97.20$\pm$13.48} &                80.17$\pm$25.85 &  97.15$\pm$13.47 &   82.31$\pm$5.79 &       75.55$\pm$28.26 \\
			                               & $\operatorname{CVaR}_{0.2}$ &                   84.76$\pm$9.74 &                  \textbf{93.57$\pm$10.66} &                27.31$\pm$28.69 &  85.14$\pm$16.82 &  66.99$\pm$11.67 &       31.02$\pm$28.35 \\
			\bottomrule
		\end{tabular}
	}
\end{table*}

In the offline setting of the Stochastic MuJoCo environment, our results suggest that some performance gains achieved with the Mean-CVaR objective may diminish in more complex scenarios. While our models generally outperform other risk-sensitive offline algorithms, such as ORAAC and CODAC, part of this improvement can be attributed to the strength of the underlying offline algorithms, like TD3BC and AWAC. For example, in Walker2D, the Mean-CVaR objective improves performance compared to TD3BC and AWAC. However, in other environments, such as Hopper, it actually decreases performance. The Mean-CVaR objective assigns different weights to various quantiles of the return distribution. While lower quantiles are crucial in risk-sensitive applications, their estimation is inherently less accurate due to the smaller
sample size used. As a result, the noisy estimation of these quantiles may lead to decreased performance in
some cases.

\begin{table*}[!ht]
	\caption{Normalized results for the offline experiments with the Medium dataset. We report the expected return and $\operatorname{CVaR}_{\alpha}$, averaged across 5 seeds. The $\pm$ symbol represents the standard deviation across seeds.}
	\label{tab:mujoco_offline_medium}
	\centering
	\resizebox{\linewidth}{!}{
		\begin{tabular}{llllllllll}
			\toprule 
			Environment                   & Metric                      & OAC-MC (Ours)                  & TD3BC-MC (Ours)               & ORAAC-iMC       & CODAC-iMC       & CQL             & IQL                     & AWAC            & TD3BC           \\
			\midrule 
			\multirow{2}{*}{HalfCheetah}  & $\mathbb{E}$                & \textbf{58.82$\pm$18.77} & 52.75$\pm$2.65          & 52.07$\pm$1.11  & 27.19$\pm$12.71 & 58.47$\pm$0.24  & 58.24$\pm$0.33          & 56.47$\pm$15.57 & 52.77$\pm$3.90  \\
			                               & $\operatorname{CVaR}_{0.2}$ & \textbf{47.63$\pm$15.96} & 38.55$\pm$13.09         & 19.15$\pm$0.95  & 14.24$\pm$7.04  & 22.03$\pm$0.41  & 21.35$\pm$1.05          & 36.54$\pm$22.84 & 43.16$\pm$2.64  \\         
			\midrule 
			\multirow{2}{*}{Hopper}       & $\mathbb{E}$                & 54.45$\pm$6.89           & 54.17$\pm$7.15          & 36.50$\pm$2.76  & 37.66$\pm$4.77  & 41.08$\pm$31.96 & \textbf{78.01$\pm$5.29} & 59.61$\pm$11.31 & 68.33$\pm$16.16 \\
			                               & $\operatorname{CVaR}_{0.2}$ & 42.50$\pm$3.48           & 43.43$\pm$3.93          & 27.38$\pm$2.01  & 26.28$\pm$5.09  & 29.48$\pm$30.04 & \textbf{67.34$\pm$3.61} & 46.42$\pm$13.27 & 56.62$\pm$14.59 \\
			\midrule 
			\multirow{2}{*}{Walker2D}     & $\mathbb{E}$                & 76.30$\pm$3.86           & \textbf{88.73$\pm$2.89} & 25.99$\pm$10.05 & 5.65$\pm$2.06   & 77.79$\pm$4.95  & 51.33$\pm$5.34          & 68.74$\pm$9.08  & 86.28$\pm$7.67  \\
			                               & $\operatorname{CVaR}_{0.2}$ & 43.08$\pm$21.66          & \textbf{81.89$\pm$4.45} & 6.47$\pm$0.55   & 3.25$\pm$2.71   & 51.98$\pm$11.82 & 3.61$\pm$2.53           & 28.00$\pm$11.98 & 69.81$\pm$20.47 \\
						
			\bottomrule
		\end{tabular}
	}
\end{table*}

\subsection{Discussion}
Our empirical results highlight the advantages of using a static risk measure over both risk-neutral baselines and risk-sensitive methods based on iterative risk formulations. We also demonstrate that our offline framework can be used to learn diverse risk-sensitive policies aligned with different risk preferences. One notable observation from our experiments is that while stochastic policies are well-suited for inherently stochastic environments, the added randomness can hinder the effectiveness of risk-sensitive policies. This trend appears consistently across both online and offline settings, where TD3-SRM and TD3BC-SRM achieve strong risk-sensitive performance in their respective domains.

Another important insight is that naively incorporating risk sensitivity into distributional actor-critic methods with entropy-based exploration can lead to degraded performance. This issue arises because soft critics estimate both future rewards and policy entropy. Since risk adjustments are meaningful only for returns and not for entropy, the risk measure should be applied solely to the reward component. Decoupling return estimation from entropy may improve the performance of risk-sensitive methods while retaining the benefits of entropy-based exploration, presenting a compelling direction for future work.

\section{Conclusion}

Despite rapid progress in risk-sensitive reinforcement learning, two important gaps remain. First, the dominant use of iterative risk measures in DRL, while easy to implement, often leads to suboptimal policies when the goal is to optimize static risk preferences. Second, despite growing interest in offline RL, there is little support for directly optimizing static risk measures, even though they are essential in safety-critical settings.

In this paper, we address these challenges by introducing a unified actor-critic framework for optimizing static SRMs in both online and offline RL. Our approach leverages the supremum representation of SRMs to develop a simple and efficient optimization method with minimal computational overhead. We provide convergence guarantees in the online tabular setting and show how policy constraints can be integrated to mitigate distributional shift in the offline setting, enabling practical use on fixed datasets.

Our framework enables learning a wide range of risk-sensitive policies from a single offline dataset, with each policy tailored to a different level of risk preference. This is especially valuable in domains like finance, healthcare, and robotics, where environment interaction is expensive or risky and the ability to manage worst-case outcomes is critical. Through extensive experiments, we show that our methods consistently outperform existing risk-sensitive algorithms, especially when it comes to controlling risk while maintaining strong returns.

Overall, our contributions bring together solid theoretical grounding and practical tools, pushing risk-sensitive RL closer to deployment in real-world applications.

\endgroup




\appendix
\renewcommand{\thetheorem}{\arabic{theorem}}

\section{Details of the Experiments}
\label{details}

\subsection{Baselines}
\label{app:baselines}
\paragraph{Online Setting.}
The implementation of our model, as well as the baselines, follow the single-file implementation of RL algorithms from CleanRL \citep{Huang.etal2022b} for clarity. We use the JAX \citep{Bradbury.etal2018} implementation of the SAC and TD3 algorithms from CleanRL\footnote{\url{https://github.com/vwxyzjn/cleanrl}}'s repository. We implement the AC algorithm with JAX according to Algorithm \ref{alg:stochastic_ac}.

\paragraph{Offline Setting.}
For our risk-sensitive baselines, we use the original implementation of the O-RAAC\footnote{\url{https://github.com/nuria95/O-RAAC}} and CODAC\footnote{\url{https://github.com/JasonMa2016/CODAC}} algorithms. For TD3BC, AWAC, CQL, and IQL, we use the JAX implementations in the JAX-CORL\footnote{\url{https://github.com/nissymori/JAX-CORL}} repository. For each algorithm, we use the default hyperparameters introduced in their corresponding paper.

The Expert-Replay datasets used in the trading and portfolio allocation environments are generated by training expert policies using SAC and TD3, respectively. For the HIV treatment environment, we use the Medium-Replay dataset provided by \citet{Rigter.etal2023}\footnote{\url{https://huggingface.co/datasets/marcrigter/1R2R-datasets}}. The Medium datasets for the Stochastic MuJoCo environments are obtained from the implementation of \citet[O-RAAC]{Urpi.etal2021a}. All datasets used in our experiments are publicly available through the project’s GitHub repository.

\subsection{Implementation Details}
\label{app:implementation}
As described in section \ref{method}, the risk-sensitive variants of the AC and TD3 algorithms in the online setting, as well as the AWAC and TD3BC algorithms in the offline settings, are developed by adding three components: 1) a distributional Critic, 2) an extended state space, and 3) a periodically updated function $h$ based on the return distribution of the initial state. We use the QR-DQN algorithm for our distributional critic and we set the number of quantiles to $50$. We also use the return-distribution of the initial state to update the function $h$ every $500$ time-steps.

Table \ref{tab:hyper} lists the shared hyperparameters for both online and offline settings. In the online experiments, the Trading environment is trained for 500,000 timesteps and the HIV Treatment environment and the Stochastic MuJoCo environment for 1,000,000 timesteps. For the offline experiments, all environments are trained for 500,000 timesteps. Additionally, the Lagrange multiplier for AWAC and OAC-SRM is fixed at 1.0, while the behavior cloning coefficient in TD3BC and TD3BC-SRM is set to 2.5.
\begin{table}[!ht]
    \caption{Default hyperparameters in different models}
    \label{tab:hyper}
    \centering
    \begin{tabular}{cc}
    \toprule
     Hyperparameter & Value \\
     \midrule
     Learning Rate & 3e-4 \\
     Discount Factor ($\gamma$) & 0.99 \\
     Batch Size & 256\\
     Number of Quantiles & 50 \\
     Target Smoothing Coefficient & 5e-3 \\
     Hidden Layer dimension & 256 \\
     Number of layers & 2 \\
     Activation Function & ReLu \\
     \bottomrule
    \end{tabular}
\end{table}

The code for the project is available at \url{https://github.com/MehrdadMoghimi/ACSRM}. 

\subsection{Evaluation}
\label{app:evaluation}
For both the online and offline settings, we train each algorithm using 5 different random seeds. Once training is complete, we generate 1,000 trajectories using the learned policies. The returns are normalized following the method outlined by \citet{Fu.etal2021}, utilizing the values from Table \ref{tab:normalization}.

\begin{table}[!ht]
    \caption{Expected return of a random policy and an expert policy after 1 million time-steps.}
    \label{tab:normalization}
    \centering
    \begin{tabular}{lcc}
    \toprule
     Environment & Random Policy & Expert Policy \\
     \midrule
     Trading & -6.17 & 1.72 \\
     Portfolio (Train) & -0.0776 & 0.0248\\
     Portfolio (Test) & -0.0815 & 0.0281 \\
     HIV Treatment & -48.7 & 5557.9 \\
     HalfCheetah & -57.95 & 659.49 \\
     Hopper & -13.37 & 1602.04  \\
     Walker2D  & -17.32 & 1790.02  \\
     
     \bottomrule
    \end{tabular}
\end{table}

\subsection{Computational Requirement}
\label{app:computation}
In the online setting, each experiment with $1$ million time-steps takes approximately $1-2$ hour on a Windows 11 PC with 16GB of RAM, Intel Core i5-12600K CPU, and Nvidia RTX 3060 GPU. In the offline setting, $1$ million time-steps of training take $10-20$ minutes on the same system.

\section{Proof of Proposition \ref{prop}}
\label{app:functionh}

As seen in Section \ref{sec:distributional}, we have $\tau_i=\frac{i}{N}$ and $\hat{\tau}_i=\frac{\tau_{i-1}+\tau_i}{2}$ for \(i=1,\cdots, N\). The quantile representation of the random variable $Z$ with \(q_i = F^{-1}_{Z}(\hat{\tau}_i)\) allows us to express Equation \ref{functionh} as a piecewise linear function:
\begin{equation}
    \tilde{h}_{\phi,Z}(z) \approx \sum_{i=1}^{N} w_i \left(q_i + \frac{1}{\hat{\tau}_i}(z - q_i)^-\right),
\end{equation}
where  \(w_i = \hat{\tau}_i \left(\phi(\tau_{i-1}) - \phi(\tau_i)\right)\). To show this, using the definition of function $h$, we can write:
\begin{align}
\label{eq:tildeh}
    h(z) & = \int_{0}^{1} \left( F_{Z}^{-1}(\alpha)+\frac{1}{\alpha}\left(z - F_{Z}^{-1}(\alpha)\right)^{-} \right) \mu(\mathrm{d}\alpha) \nonumber \\
    & = \sum_{i=1}^{N} \left(\int_{\tau_{i-1}}^{\tau_{i}} \left( F_{Z}^{-1}(\alpha) +  \frac{1}{\alpha}\left(z - F_{Z}^{-1}(\alpha)\right)^{-} \right)\mu(\mathrm{d}\alpha)\right)   \nonumber\\
    & \stackrel{(a)}{=} \sum_{i=1}^{N} \left(q_i \int_{\tau_{i-1}}^{\tau_{i}} \mu(\mathrm{d}\alpha) + \left(z - q_i\right)^{-}\int_{\tau_{i-1}}^{\tau_{i}} \frac{1}{\alpha} \mu(\mathrm{d}\alpha)\right)    \nonumber\\
    & \stackrel{(b)}{\approx} \sum_{i=1}^{N} \left(q_i \hat{\tau}_i\left(\phi(\tau_{i-1}) - \phi(\tau_{i})\right) + \left(z - q_i\right)^{-}\left(\phi(\tau_{i-1}) - \phi(\tau_{i})\right)\right)    \nonumber\\
    & \approx \sum_{i=1}^{N} w_i \left(q_i + \frac{1}{\hat{\tau}_i}(z - q_i)^-\right).
\end{align}
In this calculation, the integration interval $[0,1]$ is partitioned into $N$ subintervals: $[\tau_0, \tau_1), [\tau_1, \tau_2), \ldots, [\tau_{N-1}, \tau_N]$. Each integral $\int_{\tau_{i-1}}^{\tau_i} \mu(\mathrm{d}\alpha)$ is evaluated over $[\tau_{i-1}, \tau_i)$, meaning it includes the lower limit $\tau_{i-1}$ but excludes the upper limit $\tau_i$. In step (a), we used the fact that $q_i$ is constant within each interval $[\tau_{i-1}, \tau_i)$. In step (b), we apply the relationship between $\phi$ and $\mu$ as given in Equation \ref{eq:phimu}. More specifically, the approximation for the first term relies on the idea that if the interval is sufficiently small, it is reasonable to approximate $\alpha$ in the integrand by its midpoint value $\hat{\tau}_i$. This yields:
\[
\int_{\tau_{i-1}}^{\tau_i}\frac{\mu(\mathrm{d}\alpha)}{\alpha}\approx \frac{1}{\hat\tau_i}\int_{\tau_{i-1}}^{\tau_i}\mu(\mathrm{d}\alpha) \implies \int_{\tau_{i-1}}^{\tau_i}\mu(\mathrm{d}\alpha)
\approx \hat\tau_i\left[\phi(\tau_{i-1})-\phi(\tau_i)\right].
\]

Using the quantile value of the any state-action pair $\tilde{q}_j=F_{G^\pi(\bar{x},a)}^{-1}(\hat{\tau}_j)$ and setting \(q_i = F^{-1}_{G^\pi}(\hat{\tau}_i)\), we can estimate the Q-values with 
\begin{equation}
\label{eq:actionselection}
Q^\pi(\bar{x}, a) =  C + \frac{1}{N} \sum_{j=1}^N \left(\sum_{i=1}^N \left(\phi(\tau_{i-1}) - \phi(\tau_{i})\right)\left(s + c\tilde{q}_j - q_i\right)^{-}\right), 
\end{equation}
where $C=\int_{0}^{1} F_{G}^{-1}(\alpha) \mu(\mathrm{d}\alpha)$ is a constant for all state actions.

\section{Proof of Theorem \ref{theorem:ac_inner}}
\label{app:ac_inner}

An important result in traditional RL is the Performance Difference Lemma \citep{Kakade.Langford2002}, which states that the difference in value between two policies can be represented using the advantage function of one policy and the state occupancy measure of another. \citet{Kim.etal2024a} extend this result to show that this relationship also holds for any non-decreasing concave function. For completeness, we provide the proof of this lemma here using our notation.
\begin{lemma}[Risk-Sensitive Performance Difference Lemma]
The performance difference of policies $\pi$ and $\pi^\prime$ for a non-decreasing concave function $h$ is given by
  \begin{equation}
    \label{eq:perfdiff}
    J(\pi^\prime,h) - J(\pi,h) = \mathbb{E}_{d_{\xi_0}^{\pi^\prime}, \pi^\prime}\left[A^{\pi}_{h}(\bar{x}, a)\right]/(1-\gamma).
\end{equation}  
\end{lemma}
\begin{proof}
We can prove this lemma by writing:
\[
\begin{aligned}
& \mathbb{E}_{d_{\xi_0}^{\pi^\prime}, \pi^\prime}\left[A^{\pi}_{h}(\bar{x}, a)\right]/(1-\gamma) = \mathbb{E}_{\pi^\prime}\left[\sum_{t=0}^{\infty} \gamma^t A^{\pi}_{h}(\bar{x}_t, a_t)\right] \\
& = \mathbb{E}_{\pi^\prime}\left[\sum_{t=0}^{\infty} \left(c_t Q^{\pi}_{h}(\bar{x}_t, a_t) - c_t V^{\pi}_{h}(\bar{x}_t)\right)\right] \\
& \stackrel{(a)}{=} \mathbb{E}_{\pi^\prime}\left[\sum_{t=0}^{\infty} \left(c_t V^{\pi}_{h}(\bar{x}_{t+1}) - c_t V^{\pi}_{h}(\bar{x}_t)\right)\right] \\
& = \mathbb{E}_{\pi^\prime}\left[c_1 V^{\pi}_{h}(\bar{x}_{1}) - c_0 V^{\pi}_{h}(\bar{x}_0) + c_2 V^{\pi}_{h}(\bar{x}_{2}) - c_1 V^{\pi}_{h}(\bar{x}_{1}) + \cdots + \right] \\
& = \lim_{t\rightarrow\infty}\mathbb{E}_{\pi^\prime}\left[c_t V^{\pi}_{h}(\bar{x}_t)\right] - \mathbb{E}_{\pi^\prime}\left[c_0 V^{\pi}_{h}(\bar{x}_0)\right] \\
& = \lim_{t\rightarrow\infty}\mathbb{E}_{\pi^\prime}\left[c_t V^{\pi}_{h}(\bar{x}_t)\right] - \mathbb{E}[h(G^{\pi})] \\
& = \lim_{t\rightarrow\infty}\mathbb{E}_{\pi^\prime}\left[\mathbb{E}\left[h\left(s_t + c_t G^{\pi}(\bar{x}_t)\right)\right]\right] - \mathbb{E}[h(G^{\pi})] \\ 
& = \lim_{t\rightarrow\infty}\mathbb{E}_{\pi^\prime}\left[\mathbb{E}\left[h\left(\sum_{t^\prime=0}^{t} \gamma^{t^\prime} R^{\pi^\prime}_{t^\prime} + \gamma^t G^{\pi}(\bar{x}_t)\right)\right]\right] - \mathbb{E}[h(G^{\pi})] \\ 
& = \mathbb{E}[h(G^{\pi^\prime})]  - \mathbb{E}[h(G^{\pi})] 
\end{aligned}
\]
where $(a)$ uses the relationship between $Q$ and $V$, i.e. \(
Q^{\pi}_{h}(\bar{x}, a)= \mathbb{E}_{\bar{x}^{\prime} \sim P(\bar{x}, a)}\left[V^{\pi}_{h}\left(\bar{x}^{\prime}\right)\right]\)
\end{proof}

With the risk-sensitive performance difference lemma established, the proof of Theorem \ref{theorem:ac_inner} follows the overall structure of the convergence proof for traditional policy gradient methods in the finite state and action case \citep{Agarwal.etal2021a}. However, our analysis requires a key adaptation: we incorporate a risk-adjusted return objective rather than the standard expected return. This modification introduces additional complexity in bounding the updates, which we address by leveraging the structure of the spectral risk measure and its impact on the advantage function. We assume the rewards are positive and bounded, i.e., \( R \in [R_{\min}, R_{\max}] \) with \( R_{\min} \geq 0 \), which implies that the state and state-action values are also bounded:
\[
V^{\pi}_{h}(\bar{x}),Q^{\pi}_{h}(\bar{x}, a) \in [h(s+cG_{\min})/c, h(s+cG_{\max})/c]
\]
where $G_{\min}:=R_{\min}/(1-\gamma)$ and $ G_{\max}:=R_{\max}/(1-\gamma)$. Since $h$ is a $\phi(0)$-Lipschitz function, the upper bound of the advantage function can be calculated as:
\begin{align}
    \max |A^{\pi}_{h}(\bar{x}, a)|  & = \max |Q^{\pi}_{h}(\bar{x}, a) - V^{\pi}_{h}(\bar{x})| \nonumber \\
    & \leq \left(h(s+cG_{\max}) - h(s+cG_{\min})\right)/c \nonumber \\
    & \leq \phi(0)(G_{\max} - G_{\min}). \nonumber 
\end{align}

Furthermore, we denote the policy as \(\pi\) instead of \(\pi_h\) for brevity, even though it depends on the function \(h\). The softmax policy parameterization is given by:
\[
\pi_\theta(a \mid \bar{x}) := \dfrac{\exp(\theta(\bar{x},a))}{\sum_{a^\prime} \exp(\theta(\bar{x},a^\prime))}, \forall (\bar{x},a) \in \bar{X},A
\]
Let \( \pi_t \) denote \( \pi_{\theta_t} \). The policy can be written as:
\[
\pi_{t+1}(a \mid \bar{x}) = \pi_t(a \mid \bar{x})\dfrac{\exp(\theta_{t+1}(\bar{x},a)-\theta_t(\bar{x},a))}{Z_t(\bar{x})} 
\]
where \( Z_t(\bar{x}) := \sum_{a \in A} \pi_t(a \mid \bar{x}) \exp(\theta_{t+1}(\bar{x}, a) - \theta_t(\bar{x}, a)) \). Our goal is to show that, starting from \( \pi_0 \), we can use the policy gradient defined in Equation \ref{eq:policy_gradient} to reach the optimal policy in the inner optimization. By using the gradient updates as in the Natural Policy Gradient algorithm, i.e.:
$$
\begin{aligned}
& F(\theta)=\mathbb{E}_{\bar{x} \sim d_{\xi_0}^{\pi_\theta}, a \sim \pi_\theta(\cdot \mid \bar{x})}\left[\nabla_\theta \log \pi_\theta(a \mid \bar{x})\left(\nabla_\theta \log \pi_\theta(a \mid \bar{x})\right)^{\top}\right], \\
& \theta_{t+1}=\theta_t+\eta_t F\left(\theta_t\right)^{\dagger} \nabla_\theta J(\pi_t,h),
\end{aligned}
$$
where $F$ denotes the Fisher information matrix (induced by $\pi_\theta$) and $F^{\dagger}$ denotes the Moore-Penrose pseudoinverse of $F$, the updates of the policy parameters take a simple form:
\begin{equation}
\label{eq:theta_update}
\theta_{t+1}=\theta_t+\frac{\eta_t}{1-\gamma}\left(A_{h}^t\right),
\end{equation}
as the pseudoinverse of the Fisher information cancels out the effect of the state distribution in Equation \ref{eq:policy_gradient} \citep[Lemma 15]{Agarwal.etal2021a}. To continue, we need to prove some intermediate results. 

\begin{lemma}[Improvement lower bound]
\label{lem:lower}
For any initial state distribution ${\xi_0}$, the sequence of policies generated by \ref{eq:theta_update} follows:
\[
J(\pi_{t+1}, h)-J\left(\pi_t, h\right) \geq \frac{1-\gamma}{\eta_t} \sum_{\bar{x} \in \bar{X}} {\xi_0}(\bar{x}) \log Z_t(\bar{x})
\]
\end{lemma}
\begin{proof}
First, note that $\log Z_t(\bar{x})$ is non-negative, which results from the concavity of the logarithmic function and 
 $\sum_{a \in A} \pi_t(a \mid \bar{x}) A_{h}^t(\bar{x}, a)=0$:
\[
\begin{aligned}
    \log Z_t(\bar{x}) & = \log \sum_{a \in A} \pi_t(a \mid \bar{x}) \exp \left(\eta_t A_{h}^t(\bar{x}, a)/(1-\gamma)\right) \\
    & \geq \sum_{a \in A} \pi_t(a \mid \bar{x}) \log \exp \left(\eta_t A_{h}^t(\bar{x}, a)/(1-\gamma)\right) \\
    & = \dfrac{\eta_t}{1-\gamma} \sum_{a \in A} \pi_t(a \mid \bar{x}) A_{h}^t(\bar{x}, a)=0
\end{aligned}
\]
Then for any ${\xi_0}$:
$$
\begin{aligned}
& J(\pi_{t+1}, h)-J\left(\pi_t, h\right) \\
& =\frac{1}{1-\gamma} \sum_{\bar{x} \in \bar{X}} d_{\xi_0}^{\pi_{t+1}}(\bar{x}) \sum_{a \in A} \pi_{t+1}(a \mid \bar{x})\left(A_{h}^t(\bar{x}, a)\right) \\
& =\frac{1}{\eta_t} \sum_{\bar{x} \in \bar{X}} d_{\xi_0}^{\pi_{t+1}}(\bar{x}) \sum_{a \in A} \pi_{t+1}(a \mid \bar{x}) \log \left(\pi_{t+1}(a \mid \bar{x}) Z_t(\bar{x}) / \pi_t(a \mid \bar{x})\right) \\
& =\frac{1}{\eta_t} \sum_{\bar{x} \in \bar{X}} d_{\xi_0}^{\pi_{t+1}}(\bar{x})\left(\mathrm{D}_{\mathrm{KL}}\left(\pi_{t+1}(\cdot \mid \bar{x}) \| \pi_t(\cdot \mid \bar{x})\right)+\log Z_t(\bar{x})\right) \\
& \geq \frac{1}{\eta_t} \sum_{\bar{x} \in \bar{X}} d_{\xi_0}^{\pi_{t+1}}(\bar{x}) \log Z_t(\bar{x})  \\
& \stackrel{(\mathrm{a})}{\geq} \frac{1-\gamma}{\eta_t} \sum_{\bar{x} \in \bar{X}} {\xi_0}(\bar{x}) \log Z_t(\bar{x}). 
\end{aligned}
$$
where (a) uses the fact that $d_{\xi_0}^\pi(\bar{x})=(1-\gamma) \sum_t \gamma^t \mathbb{P}\left(\bar{x}_t=\bar{x}\right) \geq(1-\gamma) {\xi_0}(\bar{x})$ and $\log Z_t(\bar{x})\geq0$. 
\end{proof}
\begin{lemma}[Improvement upper bound]
\label{lem:upper}
The sequence of policies generated by Equation \ref{eq:theta_update} follows:
\[
J(\pi_{t+1}, h)-J\left(\pi_t, h\right) \leq \eta_t\left(\phi(0)(G_{\max} - G_{\min})\right)^2 /(1-\gamma)^2
\]
\end{lemma}
\begin{proof}

\begin{align*}
& J(\pi_{t+1}, h)-J\left(\pi_t, h\right) \\
& =\frac{1}{1-\gamma} \sum_{\bar{x}} d_{\xi_0}^{\pi_{t+1}}(\bar{x}) \sum_a \pi_{t+1}(a \mid \bar{x})\left(A_{h}^t(\bar{x}, a)\right) \\
& =\frac{1}{1-\gamma} \sum_{\bar{x}} d_{\xi_0}^{\pi_{t+1}}(\bar{x}) \sum_a\left(\pi_{t+1}(a \mid \bar{x})-\pi_t(a \mid \bar{x})\right)\left(A_{h}^t(\bar{x}, a)\right) \\
& \stackrel{(a)}{\leq} \frac{1}{1-\gamma} \sum_{\bar{x}} d_{\xi_0}^{\pi_{t+1}}(\bar{x}) \max _a\left|A_{h}^t(\bar{x}, a)\right| \sum_a\left|\pi_{t+1}(a \mid \bar{x})-\pi_t(a \mid \bar{x})\right|\\
& \stackrel{(b)}{\leq} \frac{\left\|\theta_{t+1}-\theta_t\right\|_{\infty}}{1-\gamma}\left\|A_{h}^t\right\|_{\infty} \\
& \stackrel{(c)}{=}\frac{\eta_t}{(1-\gamma)^2}\left\|A_{h}^t\right\|_{\infty}^2 \\
& \leq \frac{\eta_t}{(1-\gamma)^2} \left(\phi(0)(G_{\max} - G_{\min})\right)^2.
\end{align*}

In this derivation, (a) follows Hölder's inequality, (b) follows the fact that $\left\|\pi_{t+1}-\pi_t\right\|_{1} \leq\left\|\theta_{t+1}-\theta_t\right\|_{\infty}$ \citep[Lemma 24]{Mei.etal2020}, and (c) is resulted from policy updates in Equation \ref{eq:theta_update}. 
\end{proof}
By combining the results of Lemma \ref{lem:lower} (with the state occupancy measure of the optimal policy \( (d_{\xi_0}^{\pi^*}) \) as the initial state distribution) and Lemma \ref{lem:upper}, we obtain:
\begin{equation}
\label{eq1}
\frac{1-\gamma}{\eta_t} \sum_{\bar{x} \in \bar{X}}  d_{\xi_0}^{\pi^*}(\bar{x}) \log Z_t(\bar{x}) \leq \frac{ \eta_t}{(1-\gamma)^2}\left(\phi(0)(G_{\max} - G_{\min})\right)^2 . 
\end{equation}
Using this, we can show that for any policy \( \pi_t \) and the optimal policy \( \pi^* \), we have:

\begin{align*}
& J(\pi^*, h)-J\left(\pi_t, h\right) \\
& =\frac{1}{1-\gamma} \sum_{\bar{x} \in \bar{X}} d_{\xi_0}^{\pi^*}(\bar{x}) \sum_{a \in A} \pi^*(a \mid \bar{x})\left(A_h^t(\bar{x}, a)\right) \\
& =\frac{1}{\eta_t} \sum_{\bar{x} \in \bar{X}} d_{\xi_0}^{\pi^*}(\bar{x}) \sum_{a \in A} \pi^*(a \mid \bar{x}) \log \left(\pi_{t+1}(a \mid \bar{x}) Z_t(\bar{x}) / \pi_t(a \mid \bar{x})\right) \\
& =\frac{1}{\eta_t} \sum_{\bar{x} \in \bar{X}} d_{\xi_0}^{\pi^*}(\bar{x}) \left(\mathrm{D}_{\mathrm{KL}}\left(\pi^*(\cdot \mid \bar{x}) \| \pi_t(\cdot \mid \bar{x})\right) \right. \\
& \quad \left.-\mathrm{D}_{\mathrm{KL}}\left(\pi^*(\cdot \mid \bar{x}) \| \pi_{t+1}(\cdot \mid \bar{x})\right)+\log Z_t(\bar{x})\right)\\
& \leq \frac{1}{\eta_t} \sum_{\bar{x} \in \bar{X}} d_{\xi_0}^{\pi^*}(\bar{x})\left(\mathrm{D}_{\mathrm{KL}}\left(\pi^*(\cdot \mid \bar{x}) \| \pi_t(\cdot \mid \bar{x})\right) \right. \\
& \quad \left. -\mathrm{D}_{\mathrm{KL}}\left(\pi^*(\cdot \mid \bar{x}) \| \pi_{t+1}(\cdot \mid \bar{x})\right)\right) \\
& \quad +\frac{\eta_t}{(1-\gamma)^3}\left(\phi(0)(G_{\max} - G_{\min})\right)^2 
\end{align*}

With $K=\frac{1}{(1-\gamma)^3}\left(\phi(0)(G_{\max} - G_{\min})\right)^2$, we can write:

\begin{align*}
& \sum_{t}\left(\eta_t\left(J(\pi^*, h)-J\left(\pi_t, h\right)\right)\right) \\
& \leq \sum_{\bar{x} \in \bar{X}} d_{\xi_0}^{\pi^*}(\bar{x}) \mathrm{D}_{\mathrm{KL}}\left(\pi^*(\cdot \mid \bar{x}) \| \pi_0(\cdot \mid \bar{x})\right) \\
& \quad +\sum_{t} \frac{\eta_t^2}{(1-\gamma)^3}\left(\phi(0)(G_{\max} - G_{\min})\right)^2 \\
& = \mathrm{D}_{\mathrm{KL}}\left(\pi^* \| \pi_0\right)+ K \sum_{t} \eta_t^2 
\end{align*}
where we have 
\[
\mathrm{D}_{\mathrm{KL}}\left(\pi^* \| \pi_0\right) = \sum_{\bar{x} \in \bar{X}} d_{\xi_0}^{\pi^*}(\bar{x}) \mathrm{D}_{\mathrm{KL}}\left(\pi^*(\cdot \mid \bar{x}) \| \pi_0(\cdot \mid \bar{x})\right)
\]
for brevity. Since \( \pi_0 \) assigns positive probability to all actions, \( \mathrm{D}_{\mathrm{KL}}\left(\pi^* \| \pi_0\right) \) remains bounded. Furthermore, by the Robbins-Monro condition (\( \sum_{t} \eta_t = \infty, \sum_{t} \eta_t^2 < \infty \)), the right-hand side converges to a finite limit. Since the Robbins-Monro condition holds, we conclude that \( \lim_{t \rightarrow \infty} J\left(\pi_{t}, h\right) = J\left(\pi^*, h\right) \), which implies that the policy \( \pi_t \) converges to the optimal policy.

\section{Proof of Theorem \ref{theorem:ac_outer}}
\label{app:ac_outer}

\begin{proof}
The alternating update of the policy $\pi$ and the function $h$ can be analyzed within the Block Cyclic Coordinate Ascent framework \citep{Tseng2001, Wright2015}. We treat the policy parameters, denoted by $\theta \in \Theta$, and the quantiles of $G^{\pi}$ required to define the function $h$, denoted by $q \in Q = \left[G_{\min}, G_{\max}\right]^N$, with constraints $q_i \leq q_{i+1}, \forall i \in \{1, \cdots, N-1\}$, as two separate parameter blocks for optimizing the objective. A standard assumption used throughout the proof is that the parameter space \(\bar{\Theta} := \{\theta \in \Theta, q \in Q \mid J(\pi_\theta, h_{q}) \geq J(\pi_{\theta_0})\}\) is compact, where \(\theta_0\) represents the initial policy parameter.  Another standard assumption is \(\sup_{\theta \in \Theta}\|\nabla_\theta \log\pi_\theta\| < \infty\). Combined with the boundedness of the advantage function and the policy gradient expression in Equation \ref{eq:policy_gradient}, this ensures that \(J(\pi_\theta, h_q)\) is Lipschitz continuous with respect to \(\theta\). Importantly, since this property does not depend on the specific quantiles defining \(h_q\), we can conclude that \(J(\pi_\theta)\) is also Lipschitz continuous with respect to \(\theta\).

First, we observe that monotonic policy improvement in this context can be demonstrated as follows: 
\[
\begin{aligned}
J(\pi_{k+1}) & = \max_{h \in \mathcal{H}^\prime} J(\pi_{k+1},h)\\
& \geq \mathbb{E}\left[h_{k+1}\left(G^{\pi_{k+1}}\right)\right] \\
& \stackrel{(a)}{\geq} \mathbb{E}\left[h_{k+1}\left(G^{\pi_{k}}\right)\right] \\
& \stackrel{(b)}{=}  J(\pi_{k}) \\
\end{aligned}
\]
Here, $(a)$ follows from the fact that $\pi_{k+1}$ is the maximizing policy when function $h_{k+1}$ is used. $(b)$ also follows from the fact that $h_{k+1}$ is constructed using the quantiles of $G^{\pi_{k}}$ to maximize $\mathbb{E}\left[h\left(G^{\pi_{k}}\right)\right]$. The argument for parameterized policies follows similarly and shows that $J(\pi_{\theta_{k+1}}) \geq J(\pi_{\theta_k})$. Since $J(\pi_\theta)$ is bounded from above, the sequence $\{J(\pi_{\theta_k})\}_{k=1,\cdots}$ converges to some $J^*$.

The proof of the theorem’s final result follows closely from Proposition 2 in \citet{Zhang.etal2021b}. However, unlike their setting, our quantile vector $q$ is defined implicitly through an inner optimization, and thus requires an adaptation of their argument to account for the subgradient structure of $h_q$. Using Theorem 4.1(c) of \citet{Tseng2001}, we know that for any convergent subsequence \(\left\{\left(\theta_k, q_k\right)\right\}_{k \in \mathcal{K}}\), the limit \(\left(\theta_{\mathcal{K}}, q_{\mathcal{K}}\right)\) satisfies \(\nabla_\theta J\left(\pi_{\theta_{\mathcal{K}}}, h_{q_{\mathcal{K}}}\right) = 0\) and \(\nabla_{q} J\left(\pi_{\theta_{\mathcal{K}}}, h_{q_{\mathcal{K}}}\right)=\nabla_{q}\mathbb{E}\left[h_{q}\left(G^{\pi_{\theta_{\mathcal{K}}}}\right)\left.\right|_{q=q_{\mathcal{K}}}\right]=0\). Our objective is to show that \(\nabla_\theta J\left(\pi_{\theta_{\mathcal{K}}}\right) = \nabla_\theta J\left(\pi_{\theta_{\mathcal{K}}}, h_{q_{\mathcal{K}}}\right) = 0\), which confirms that the subsequence \(\left\{\theta_k\right\}_{k \in \mathcal{K}}\) converges to a stationary point of \(J(\pi_\theta)\).

Since \(h_q(z)\) is not differentiable at \(z = q_i\), the condition \(\nabla_{q_i} \mathbb{E}\left[h_q\left(G^{\pi_{\theta_{\mathcal{K}}}}\right)\right] = 0\), interpreted in the sense of the subgradient containing zero, implies that \(q_i\) is a \(\hat{\tau}_i\)-quantile of \(G^{\pi_{\theta_{\mathcal{K}}}}\) for non-zero $w_i$. This means that \(F_{G^{\pi_{\theta_{\mathcal{K}}}}}(q_i^-) \leq \hat{\tau}_i \leq F_{G^{\pi_{\theta_{\mathcal{K}}}}}(q_i)\). To see why, note that \(G^{\pi_{\theta_{\mathcal{K}}}}\) follows an \(N\)-quantile distribution with quantile locations represented as \(g_{\mathcal{K}} = [g_1, g_2, \dots, g_N]\), where \(F^{-1}_{G^{\pi_{\theta_{\mathcal{K}}}}}(\tau) = g_{\lceil \tau N \rceil}\). Additionally, recall the structure of \(h_q(z)\):
\[
h_q(z) = \sum_{i=1}^{N} w_i \left(q_i + \frac{1}{\hat{\tau}_i}(z - q_i)^-\right).
\]
The subgradient of \(h_q(z)\) with respect to \(q_i\) is given by:
\[
\partial_{q_i} h_q(z) =
\begin{cases} 
w_i \left(1 - \frac{1}{\hat{\tau}_i}\right), & z < q_i, \\
w_i \left[1 - \frac{1}{\hat{\tau}_i}, 1\right], & z = q_i, \\
w_i, & z > q_i.
\end{cases}
\]
For any random variable \(Z\), the subgradient of \(\mathbb{E}[h_q(Z)]\) becomes:
\begin{align*}
\partial_{q_i} \mathbb{E}\left[h_q \left(Z\right)\right] & =  \mathbb{E}\left[\partial_{q_i} h_q \left(Z\right)\right] \nonumber \\ 
 & =  w_i\left(\mathbb{P}(Z < q_i) \cdot (1 - \frac{1}{\hat{\tau}_i}) + \mathbb{P}(Z > q_i) \cdot 1  \right. \\
 & \quad \left. + \mathbb{P}(Z = q_i)\cdot [1 - \frac{1}{\hat{\tau}_i},1]\right) \nonumber\\
 & =  w_i\left(1 - \mathbb{P}(Z = q_i) - \mathbb{P}(Z < q_i) \cdot (\frac{1}{\hat{\tau}_i}) \right. \\ 
 & \quad \left.+ \mathbb{P}(Z = q_i)\cdot [1 - \frac{1}{\hat{\tau}_i},1]\right) \nonumber\\
 & =  w_i\left[1 - \mathbb{P}(Z \leq q_i) \cdot (\frac{1}{\hat{\tau}_i}), 1 - \mathbb{P}(Z < q_i) \cdot (\frac{1}{\hat{\tau}_i}) \right] \nonumber\\
& = w_i\left(1 - \frac{1}{\hat{\tau}_i}\cdot\left[\mathbb{P}(Z < q_i), \mathbb{P}(Z \leq q_i)\right]\right)  \nonumber\\
& = w_i\left(1 - \frac{1}{\hat{\tau}_i}\cdot\left[F_{Z}(q_i^-), F_{Z}(q_i)\right]\right).    
\end{align*}

For non-zero \(w_i\), the condition 
\[
0 \in w_i \left(1 - \frac{1}{\hat{\tau}_i} \cdot [F_{G^{\pi_{\theta_{\mathcal{K}}}}}(q_i^-), F_{G^{\pi_{\theta_{\mathcal{K}}}}}(q_i)]\right)
\] ensures that \(\hat{\tau}_i \in [F_{G^{\pi_{\theta_{\mathcal{K}}}}}(q_i^-), F_{G^{\pi_{\theta_{\mathcal{K}}}}}(q_i)]\), which implies \(q_i\) is a \(\hat{\tau}_i\)-quantile of \(G^{\pi_{\theta_{\mathcal{K}}}}\). Consequently, for non-zero \(w_i\), the quantiles of \({q}_{\mathcal{K}}\) match with those of \({g}_{\mathcal{K}}\) and therefore, we can conclude that \(\nabla_\theta J\left(\pi_{\theta_{\mathcal{K}}}\right) = \nabla_\theta J\left(\pi_{\theta_{\mathcal{K}}}, h_{q_{\mathcal{K}}}\right) = 0\).

Ultimately, to establish the existence of a convergent subsequence, we need to prove that \(\Theta_0 := \{\theta \in \Theta \mid J(\pi_\theta) \geq J(\pi_{\theta_0})\}\) is compact. Let \(\{\theta^i\}_{i=1, \ldots}\) be a sequence within \(\Theta_0\) that converges to a limit \(\theta^{\infty}\). For each \(i\), define \(q^i := \arg \max_{q} J(\pi_{\theta^i}, q)\). Since \(J(\pi_\theta)\) is Lipschitz continuous with respect to \(\theta\), the sequence \((\theta^i, q^i)\) must converge to \((\theta^{\infty}, q^{\infty})\).  The compactness of \(\bar{\Theta}\) guarantees that \(J(\pi_{\theta^{\infty}}, h_{q^{\infty}}) \geq J(\pi_{\theta_0})\), given that \(J(\pi_{\theta^i}, h_{q^i}) = J(\pi_{\theta^i}) \geq J(\pi_{\theta_0})\) for all \(i\). This implies that \(J(\pi_{\theta^{\infty}}) \geq J(\pi_{\theta_0})\), confirming that \(\theta^{\infty} \in \Theta_0\) and \(\Theta_0\) is compact. This completes the proof.
\end{proof}

\section{Proof of Theorem \ref{theorem:awac_srm}}
\label{app:awac_srm}
\begin{proof}
The proof of this theorem closely follows the results of \citet{Peng.etal2019} and \citet{Nair.etal2021}, which derive policy updates under a KL divergence constraint. Our formulation adapts this idea to the risk-sensitive setting by incorporating the risk-adjusted advantage $A_h^{\pi}(\bar{x},a)$, and we treat the Lagrange multiplier $\lambda$ as a tunable hyperparameter. This adaptation enables us to retain the tractability of the update rule while aligning it with the underlying risk-sensitive objective. Our formulation thus generalizes the advantage-weighted actor-critic framework to account for risk preferences in the offline policy learning.

Suppose we write the inner optimization as a search to find a policy that maximizes the expected improvement $I_h(\pi)=(1-\gamma)(J(\pi,h)-J(\pi_{k},h))$. This expected improvement lets us use the risk-sensitive performance difference lemma from the previous section to write the objective involving the advantage function. As in the proof of Theorem \ref{theorem:ac_inner}, we omit the subscript \(h\) from \(\pi_h\) for brevity.

\[
I_h(\pi)=\mathbb{E}_{\bar{x} \sim d_\pi(\bar{x})} \mathbb{E}_{a \sim \pi(a \mid \bar{x})}\left[A_h^{\pi_{k}}(\bar{x}, a)\right]
\]
Following \citet{Schulman.etal2015} and assuming that each policy $\pi_k$ is close to $\pi_\beta$ in term of the KL-divergence, we can estimate $I_h(\pi)$ by using the state distribution of $\pi_\beta$:
$$
\hat{I}_h(\pi)=\mathbb{E}_{\bar{x} \sim d_{\pi_\beta}(\bar{x})} \mathbb{E}_{a \sim \pi(a \mid \bar{x})}\left[A_h^{\pi_{k}}(\bar{x}, a)\right] .
$$
With this change, we derive the following constrained inner optimization problem: 
$$
\begin{aligned}
\underset{\pi}{\arg \max } & \int_{\bar{x}} d_{\pi_\beta}(\bar{x}) \int_{a} \pi(a \mid \bar{x})A_h^{\pi_{k}}(\bar{x}, a) d a d \bar{x} \\
\text { s.t. } & \int_{\bar{x}} d_{\pi_\beta}(\bar{x}) \mathrm{D}_{\mathrm{KL}}(\pi(\cdot \mid \bar{x}) \| {\pi_\beta}(\cdot \mid \bar{x})) d \bar{x} \leq \epsilon .\\
& \int_{a} \pi(a \mid \bar{x}) d a=1, \quad \forall \bar{x} .
\end{aligned}
$$
The Lagrangian of this optimization problem can be written as:

\begin{align*}
\mathcal{L}_h(\pi, \lambda, \alpha)=&\int_{\bar{x}} d_{\pi_\beta}(\bar{x}) \int_{a} \pi(a \mid \bar{x})A_h^{\pi_{k}}(\bar{x}, a) d a d \bar{x} \\
& \quad + \lambda\left(\epsilon-\int_{\bar{x}} d_{\pi_\beta}(\bar{x}) \mathrm{D}_{\mathrm{KL}}(\pi(\cdot \mid \bar{x}) \| {\pi_\beta}(\cdot \mid \bar{x})) d \bar{x}\right), \\
& \quad + \int_{\bar{x}} \alpha_{\bar{x}}\left(1-\int_{a} \pi(a \mid \bar{x}) d a\right) d \bar{x},
\end{align*}

where $\lambda$ and $\alpha=\left\{\alpha_{\bar{x}} \mid \forall \mathrm{\bar{x}} \in \mathcal{\bar{X}}\right\}$ are the Lagrange multipliers. The derivative of $\mathcal{L}(\pi, \lambda, \alpha)$ w.r.t $\pi(a \mid \bar{x})$ is
\begin{align*}
\frac{\partial \mathcal{L}}{\partial \pi(a \mid \bar{x})} & =d_{\pi_\beta}(\bar{x})A_h^{\pi_{k}}(\bar{x}, a)-\lambda d_{\pi_\beta}(\bar{x}) \log \pi(a \mid \bar{x}) \\
& \quad +\lambda d_{\pi_\beta}(\bar{x}) \log \pi_\beta(a \mid \bar{x})-\lambda d_{\pi_\beta}(\bar{x})-\alpha_{\bar{x}},
\end{align*}
and setting this derivative to zero and solving for $\pi(a \mid \bar{x})$ give rise to
\[
\log \pi(a \mid \bar{x})=\frac{1}{\lambda}A_h^{\pi_{k}}(\bar{x}, a)+\log \pi_\beta(a \mid \bar{x})-1-\frac{1}{d_{\pi_\beta}(\bar{x})} \frac{\alpha_{\bar{x}}}{\lambda}.
\]
Therefore, we have:
\[
\pi(a \mid \bar{x})=\pi_\beta(a \mid \bar{x}) \exp \left(\frac{1}{\lambda}A_h^{\pi_{k}}(\bar{x}, a)\right) \exp \left(-\frac{1}{d_{\pi_\beta}(\bar{x})} \frac{\alpha_{\bar{x}}}{\lambda}-1\right).
\]

With $Z(\bar{x})=\exp \left(\frac{1}{d_{\pi_\beta}(\bar{x})} \frac{\alpha_{\bar{x}}}{\lambda}+1\right)$ as the partition function that normalizes the conditional action distribution, the optimal policy can be written as
$$
\pi^*(a \mid \bar{x})=\frac{1}{Z(\bar{x})} \pi_\beta(a \mid \bar{x}) \exp \left(\frac{1}{\lambda}A_h^{\pi_{k}}(\bar{x}, a)\right).
$$
Since $\int_{a^\prime} \pi(a^\prime \mid \bar{x}) d a^\prime=1$, we also have
$$
Z(\bar{x})=\int_{a^{\prime}} \pi_\beta\left(a^{\prime} \mid \bar{x}\right) \exp \left(\frac{1}{\lambda}A_h^{\pi_{k}}(\bar{x}, a^\prime)\right) d a^{\prime} .
$$
Finally, since we use function approximation to represent the optimal policy $\pi^*$, we can project the optimal policy into the parameterized policy spaces with the following problem:
\begin{align*}
& \pi_{k+1} = \underset{\pi}{\arg \min } \mathbb{E}_{\bar{x} \sim d_{\pi_\beta}(\bar{x})}  \left[\mathrm{D}_{\mathrm{KL}}\left(\pi^*(\cdot \mid \bar{x}) \| \pi(\cdot \mid \bar{x})\right)\right] \\
& =\underset{\pi}{\arg \min }  \mathbb{E}_{\bar{x} \sim d_{\pi_\beta}(\bar{x})}\left[\mathrm{D}_{\mathrm{KL}}\left(\frac{1}{Z(\bar{x})} \pi_\beta(a \mid \bar{x}) \exp \left(\frac{1}{\lambda}A_h^{\pi_{k}}(\bar{x}, a)\right) \| \pi(\cdot \mid \bar{x})\right)\right] \\
& =\underset{\pi}{\arg \max } \mathbb{E}_{\bar{x} \sim d_{\pi_\beta}(\bar{x})} \mathbb{E}_{a \sim \pi_\beta(a \mid \bar{x})}\left[\log \pi(a \mid \bar{x}) \exp \left(\frac{1}{\lambda}A_h^{\pi_{k}}(\bar{x}, a)\right)\right] \\
\end{align*}

\end{proof}

\newpage
\section{Algorithm for OAC-SRM}
\label{app:stochastic_oac}

\begin{algorithm}[!ht]
\caption{OAC-SRM}
\label{alg:stochastic_oac}
\SetAlgoLined
\DontPrintSemicolon
\SetKwInOut{Input}{Input}
\SetKwInOut{Initialize}{Initialize}
{\small
\Input{
Batch size $M$,
Number of quantiles $N$, 
Policy update frequency $d$,
Target smoothing coefficient $\nu$,
Number of policy updates $T_{inner}$,
Number of function $h$ updates $T_{outer}$, 
Dataset $\mathcal{D}$,
Lagrange multiplier $\lambda$}

\Initialize{
Critic networks $G_{\theta_1}, G_{\theta_2}$ and Actor network $\pi_{\theta}=\mathcal{N}(f_\theta,\sigma)$ or $\pi_{\theta}=\operatorname{Categorical}(\operatorname{softmax}(f_\theta))$ with random parameters $\theta_1, \theta_2, \theta$, 
Target network parameters $\theta_1^\prime \leftarrow \theta_1, \theta_2^\prime \leftarrow \theta_2, \theta^\prime \leftarrow \theta$, 
}
\For{$1$ \KwTo $T_{outer}$}{
    \tcp{Update risk function}
    \(
    h =  \tilde{h}_{\phi, G_{\theta_1}(x_0,\pi_{\theta}(x_0))}
    \)
    
    \For{$t=1$ \KwTo $T_{inner}$}{
        \tcp{Update Critic}
        Sample mini-batch of $M$ transitions $(\bar{x}, a, r, \bar{x}^\prime)$ from $\mathcal{D}$\\
        \ForEach{transition in the mini-batch}{
            Sample target action: 
            \[
            a^\prime \sim \pi_{\theta^\prime}(\bar{x}^\prime)
            \]
            
            Compute Q-values for $k = 1, 2$:
            \[
            Q_{k} = \mathbb{E}\left[h\left(s^\prime + c^\prime G_{\theta^\prime_k}(\bar{x}^\prime, a^\prime)\right)\right]/c^\prime
            \]
            
            Select target quantile set:
            \[
            G^\prime(\bar{x}^\prime, a^\prime) = 
            \begin{cases}
            G_{\theta^\prime_1}(\bar{x}^\prime, a^\prime) & \text{if } Q_{1} \leq Q_{2} \\
            G_{\theta^\prime_2}(\bar{x}^\prime, a^\prime) & \text{otherwise}
            \end{cases}
            \]
            
            Compute target quantiles: 
            \[
            Y(\bar{x},a) = r + \gamma G^\prime(\bar{x}^\prime, a^\prime)
            \]

        }
        Minimize quantile regression loss $\mathcal{L}(G_{\theta_k}, Y), k = 1, 2$ (Equation \ref{huberloss}) for the mini-batch
        
        \tcp{Update Actor}
        \If{$t \bmod d = 0$}{
            Compute the advantage function:
            \[
            A(\bar{x},a) = Q_1(\bar{x},a) - \mathbb{E}_{\tilde{a}\sim\pi_{\theta^\prime}}\left[Q_1(\bar{x},\tilde{a})\right]
            \]
            
            Update $\theta$ using the advantage function and the policy gradient in Equation \ref{eq:oac_pg}
            
            Update target networks:
            \[
            \theta^\prime_k \leftarrow \nu \theta_k + (1 - \nu) \theta^\prime_k, k=1,2 \quad \theta^\prime \leftarrow \nu \theta + (1 - \nu) \theta^\prime
            \]
        }
    }
}
}
\end{algorithm}

\newpage
\section{Algorithm for TD3-SRM and TD3BC-SRM}
\label{app:algodeterministic}

\begin{algorithm}[!ht]
\caption{{\color{blue}TD3-SRM} \& {\color{red}TD3BC-SRM}}
\label{alg:deterministic}
\SetAlgoLined
\DontPrintSemicolon
\SetKwInOut{Input}{Input}
\SetKwInOut{Initialize}{Initialize}
{\small
\Input{
Batch size $M$,
Number of quantiles $N$, 
Noise parameters $\sigma, \tilde{\sigma}$, 
Clipping constant $c$, 
Policy update frequency $d$,
Target smoothing coefficient $\nu$,
Number of policy Updates $T_{inner}$,
Number of function $h$ updates $T_{outer}$, 
{\color{red} Dataset $\mathcal{D}$,
Behavior cloning parameter $\lambda$}
}
\Initialize{
Critic networks $G_{\theta_1}, G_{\theta_2}$ and Actor network $\pi_{\theta}=\mathcal{N}(f_\theta,\sigma)$, with random parameters $\theta_1, \theta_2, \theta$, 
Target network parameters $\theta_1^\prime \leftarrow \theta_1, \theta_2^\prime \leftarrow \theta_2, \theta^\prime \leftarrow \theta$, 
{\color{blue} Dataset $\mathcal{D}$}
}
\For{$1$ \KwTo $T_{outer}$}{
    \tcp{Update risk function}
    \(
    h =  \tilde{h}_{\phi, G_{\theta_1}(x_0,\pi_{\theta}(x_0))}
    \)
    
    \For{$t=1$ \KwTo $T_{inner}$}{
        \tcp{Collect New Data}
        {\color{blue}
        Observe state $\bar{x}$, select action $a = \pi_\theta(\bar{x}) + \epsilon$, where $\epsilon \sim \mathcal{N}(0, \sigma)$\\
        Execute $a$, observe reward $r$ and next state $\bar{x}^\prime$\\
        Store transition $(\bar{x}, a, r, \bar{x}^\prime)$ into $\mathcal{D}$\\
        }
        \tcp{Update Critic}
        Sample mini-batch of $M$ transitions $(\bar{x}, a, r, \bar{x}^\prime)$ from $\mathcal{D}$\\
        \ForEach{transition in the mini-batch}{
            Compute noisy target action: 
            \[
            a^\prime = \pi_{\theta^\prime}(\bar{x}^\prime) + \epsilon^\prime, \epsilon^\prime \sim \text{clip}(\mathcal{N}(0, \tilde{\sigma}), -c, c)
            \]
            
            Compute Q-values for $k=1,2$:
            \[
            Q_{k} = \mathbb{E}\left[h\left(s^\prime + c^\prime G_{\theta^\prime_k}(\bar{x}^\prime, a^\prime)\right)\right]/c^\prime
            \]
            
            Select target quantile set:
            \[
            G^\prime(\bar{x}^\prime, a^\prime) = 
            \begin{cases}
            G_{\theta^\prime_1}(\bar{x}^\prime, a^\prime) & \text{if } Q_{1} \leq Q_{2} \\
            G_{\theta^\prime_2}(\bar{x}^\prime, a^\prime) & \text{otherwise}
            \end{cases}
            \]
            
            Compute target quantiles: $Y(\bar{x},a) = r + \gamma G^\prime(\bar{x}^\prime, a^\prime)$
        }
        Minimize quantile regression loss $\mathcal{L}(G_{\theta_k}, Y), k = 1, 2$ (Equation \ref{huberloss}) for the mini-batch
        
        \tcp{Update Actor}
        \If{$t \bmod d = 0$}{
            Update $\theta$ using the Q values and the policy gradient in Equation {\color{blue}\ref{eq:td3srm_pg}} or {\color{red}\ref{eq:td3bcsrm_pg}}
            
            Update target networks:
             \[
            \theta^\prime_k \leftarrow \nu \theta_k + (1 - \nu) \theta^\prime_k, k = 1, 2 \quad \theta^\prime \leftarrow \nu \theta + (1 - \nu) \theta^\prime
            \]
        }
    }
}
}
\end{algorithm}

\bibliographystyle{plainnat}
\bibliography{MyLibrary.bib}

\end{document}